\documentclass{article}

\usepackage{microtype}
\usepackage{graphicx}
\usepackage{subfigure}
\usepackage{booktabs} 

\usepackage{hyperref}

\usepackage[colorinlistoftodos]{todonotes}
\usepackage{amsfonts,amsmath,amsthm,amssymb}
\usepackage{bm}
\usepackage{multirow}
\usepackage{color}
\usepackage{subfigure}
\usepackage{wrapfig}

\newtheorem{theorem}{Theorem}

\newtheorem{lemma}{Lemma}

\DeclareMathOperator*{\argmax}{arg\,max}

\newcommand{\algo}{{\small \sc\textsf{SNO-MDP}}}
\newcommand{\es}{{\small \sc\textsf{ES}}}
\newcommand{\pes}{{\small \sc\textsf{P-ES}}}
\newcommand{\gpsafetygym}{{\small \sc\textsf{GP-Safety-Gym}}}

\usepackage{xcolor}

\usepackage[accepted]{icml2020}

\icmltitlerunning{Safe Reinforcement Learning in Constrained Markov Decision Processes}

\begin{document}

\twocolumn[
\icmltitle{Safe Reinforcement Learning in Constrained Markov Decision Processes}




\begin{icmlauthorlist}
\icmlauthor{Akifumi Wachi}{ibm}
\icmlauthor{Yanan Sui}{tsinghua}
\end{icmlauthorlist}

\icmlaffiliation{ibm}{IBM Research AI, Tokyo, Japan}
\icmlaffiliation{tsinghua}{Tsinghua University, Beijing, China}

\icmlcorrespondingauthor{Akifumi Wachi}{akifumi.wachi@ibm.com}

\icmlcorrespondingauthor{Yanan Sui}{ysui@tsinghua.edu.cn}

\icmlkeywords{Machine Learning, ICML}

\vskip 0.3in
]



\printAffiliationsAndNotice{}  

\begin{abstract}
Safe reinforcement learning has been a promising approach for optimizing the policy of an agent that operates in safety-critical applications.
In this paper, we propose an algorithm, \algo, that explores and optimizes Markov decision processes under unknown safety constraints.
Specifically, we take a stepwise approach for optimizing safety and cumulative reward. 
In our method, the agent first learns safety constraints by expanding the safe region, and then optimizes the cumulative reward in the certified safe region.
We provide theoretical guarantees on both the satisfaction of the safety constraint and the near-optimality of the cumulative reward under proper regularity assumptions.
In our experiments, we demonstrate the effectiveness of \algo\ through two experiments: one uses a synthetic data in a new, openly-available environment named \gpsafetygym, and the other simulates Mars surface exploration by using real observation data.
\end{abstract}

\section{Introduction}

In many real applications, environmental hazards are first detected \textit{in situ}.
For example, a planetary rover exploring Mars does not obtain high-resolution images at the time of its launch.
In usual cases, after landing on Mars, the rover takes close-up images or observes terrain data.
Leveraging the acquired data, ground operators identify whether each position is safe.
Hence, for fully automated operation, an agent must autonomously \textit{explore} the environment and \textit{guarantee} safety.

In most cases, however, guaranteeing safety (i.e., surviving) is \textit{not} the primary objective.
The optimal policy for ensuring safety is often extremely conservative (e.g., stay at the current position).
Even though avoiding hazards is an essential requirement, the primary objective is nonetheless to obtain rewards (e.g., scientific gain).

As a framework to solve this problem, safe reinforcement learning (safe RL, \citet{garcia2015comprehensive}) has recently been noticed by the research community.
The objective of safe RL is to maximize the cumulative reward while guaranteeing or encouraging safety.
Especially in problem settings in which the reward and safety functions are \textit{unknown a priori}, however, a great deal of previous work (e.g., \citet{wachi2018safe}) theoretically guarantees the satisfaction of the safety constraint, but the acquired policy is not necessarily near-optimal in terms of the cumulative reward.
In this paper, we propose a safe RL algorithm that guarantees a near-optimal cumulative reward while guaranteeing the satisfaction of the safety constraint as well.

\paragraph{Related work.}

Conventional reinforcement learning literature has been agnostic with respect to safety, while pursuing efficiency and optimality of the cumulative reward.
Representatives of such work are \textit{probably approximately correct Markov decision process (PAC-MDP)} algorithms \citep{brafman2002r,kearns2002near,strehl2006pac}.
Algorithms with the PAC-MDP property enable an agent to learn a near-optimal behavior with a polynomial number of samples.
In addition, \citet{kolter2009near} and \citet{araya2012near} proposed algorithms to obtain an $\epsilon$-close solution to the Bayesian optimal policy.

As the research community tries to apply RL algorithms to real-world systems, however, safety issues have been highlighted.
RL algorithms inherently require an agent to explore unknown state-action pairs, and algorithms that are agnostic with respect to safety may execute unsafe actions without deliberateness.
Hence, it is important to develop algorithms that guarantee safety even during training, at least with high probability.

A notable approach is \textit{safe model-based RL} \citep{berkenkamp2017safe,fisac2018general}.
In this domain, safety is associated with a state constraint; thus, the resulting algorithm is well suited for such contexts as a drone learning how to hover.
The parameters of a drone are not perfectly known a priori, but we have prior knowledge on what states are unsafe (e.g., a pitch angle of more than 50 degrees is unsafe).
In the field of control theory, constrained model predictive control \cite{mayne2000constrained} has been popular. 
For example, \citet{aswani2013provably} proposed an algorithm for guaranteeing robust feasibility and constraint satisfaction for a learned model using constrained model predictive control.

On the other hand, safe \textit{model-free} RL has also been successful, especially in continuous control tasks.
For example, \citet{achiam2017constrained} proposed the constrained policy optimization (CPO) algorithm while guaranteeing safety in terms of constraint satisfaction.
Moreover, \citet{chow2019lyapunov} leveraged Lyapunov functions to learn policies with high expected cumulative reward, while guaranteeing the satisfaction of safety constraints. 

Finally, several previous studies have addressed how to explore a safe space in an environment that is unknown a priori \citep{sui2015safe,turchetta2016safe}.
This type of problem setting is well-suited for cases such as a robot exploring an uncertain environment (e.g., a planetary surface, a disaster site).
In particular, under the safety constraint, \citet{sui2018stagewise} proposed a stepwise algorithm for finding the maximum value of the reward function in a state-less setting (i.e., the bandit problem), while \citet{wachi2018safe} proposed an algorithm for maximizing the cumulative reward in an MDP setting (i.e., the planning problem). 

\paragraph{Our contributions.}

We propose a safe near-optimal MDP, \algo\ algorithm, for achieving a near-optimal cumulative reward while guaranteeing safety in a constrained MDP.
This algorithm first explores the safety function and then optimizes the cumulative reward in the certified safe region.
We further propose an algorithm called Early Stopping of Exploration of Safety (\es$^2$) to achieve faster convergence while maintaining probabilistic guarantees with respect to both safety and reward.
We examine \algo\ by applying PAC-MDP analysis and prove that, with high probability, the acquired policy is near-optimal with respect to the cumulative reward while guaranteeing safety.
We build an openly-available test-bed called \gpsafetygym\ for synthetic experiments.\footnote{\url{https://github.com/akifumi-wachi-4/safe_near_optimal_mdp}}
The safety and efficiency of \algo\ are then evaluated with two experiments: one in the \gpsafetygym\ synthetic environment, and the other using real Mars terrain data.

\section{Problem Statement}

A safety constrained MDP is defined as a tuple
\[
\mathcal{M} = \langle \mathcal{S}, \mathcal{A}, f, r, g, \gamma \rangle,
\]
where $\mathcal{S}$ is a finite set of states $\{\bm{s}\}$, $\mathcal{A}$ is a finite set of actions $\{a\}$, $f: \mathcal{S} \times \mathcal{A} \rightarrow \mathcal{S}$ is a deterministic state transition function, $r: \mathcal{S} \rightarrow (0, R_{\max}]$ is a bounded reward function, $g: \mathcal{S} \rightarrow \mathbb{R}$ is a safety function, and $\gamma \in \mathbb{R}$ is a discount factor.
We assume that both the reward function $r$ and the safety function $g$ are \textit{not known a priori}. At every time step $t \in \mathbb{N}$, the agent must be in a ``safe'' state.
More concretely, for a state $\bm{s}_t$, the safety function value $g(\bm{s}_t)$ must be above a threshold $h \in \mathbb{R}$; that is, the safety constraint is represented as $g(\bm{s}_t) \ge h$.

A policy $\pi: \mathcal{S} \rightarrow \mathcal{A}$ maps a state to an action.
The value of a policy is evaluated according to the discounted cumulative reward under the safety constraint.
Let $V_{\mathcal{M}}$ denote the value function in the MDP, $\mathcal{M}$.
In summary, we represent our problem as follows:
\begin{alignat*}{2}
\quad
\text{maximize:} &\quad V_{\mathcal{M}}^{\pi}(\bm{s}_t) = \mathbb{E}\left[\ \sum_{\tau=0}^\infty \gamma^{\tau} r(\bm{s}_{t+\tau})\ \middle|\ \bm{s}_t \right] \\
\text{subject to:} &\quad g(\bm{s}_{t+\tau}) \ge h, \quad \forall \tau = [0, \infty].
\end{alignat*}

\paragraph{Difficulties.}

In conventional safety-constrained RL algorithms, the safety function is assumed to be known a priori.
The key difference lies in the fact that we need to explore a safety function that is unknown a priori while guaranteeing satisfaction of the safety constraint.

However, it is intractable to solve the above problem without further assumptions.
First of all, without prior information on the state-and-action pairs known to be safe, an agent cannot take any viable action at the very beginning.
Second, if the safety function does not exhibit any regularity, then the agent cannot infer the safety of decisions.

\paragraph{Assumptions.}

To overcome the difficulties mentioned above, we adopt two assumptions from \citet{sui2015safe} and \citet{turchetta2016safe}.
For the first difficulty, we simply assume that the agent starts in an initial set of states, $S_0$, that is known a priori to be safe.
Second, we assume regularity for the safety function.
Formally speaking, we assume that the state space $\mathcal{S}$ is endowed with a positive definite kernel function, $k^g$, and that the safety function has a bounded norm in the associated reproducing kernel Hilbert space (RKHS, \citet{scholkopf2001learning}).
The kernel function, $k^g$ is employed to capture the regularity of the safety function.
Finally, we further assume that the safety function $g$ is $L$-Lipschitz continuous with respect to some distance metric $d(\cdot, \cdot)$ on $\mathcal{S}$.

As with the safety function, we also assume that the reward function has a bounded norm in the associated RKHS, and that its regularity is captured by another positive definite kernel function, $k^r$.

The above assumptions allow us to characterize the reward and safety functions by using Gaussian processes (GPs, see \citet{rasmussengaussian}).
By using the GP models, the values of $r$ and $g$ at unobserved states are predicted according to previously observed functions' values.
An advantage of leveraging GPs is that we can obtain both optimistic and pessimistic measurements of the two functions by using the inferred means and variances.
A GP is specified by its mean, $\mu(\bm{s})$, and covariance, $k(\bm{s}, \bm{s}')$. The reward and safety functions are thus modeled as
\begin{alignat*}{2}
r(\bm{s}) &= \mathcal{GP}(\mu^r(\bm{s}), k^r(\bm{s}, \bm{s}')), \\
g(\bm{s}) &= \mathcal{GP}(\mu^g(\bm{s}), k^g(\bm{s}, \bm{s}')).
\end{alignat*}
Without loss of generality, let $\mu(\bm{s}) = 0$ for all $\bm{s} \in \mathcal{S}$.
For the reward and safety functions, we respectively model the observation noise as $y^r = r(\bm{s}) + n^r$ and $y^g = g(\bm{s}) + n^g$, where $n^r \sim \mathcal{N}(0, \sigma^2_r)$ and $n^g \sim \mathcal{N}(0, \sigma^2_g)$.
The posteriors over $r$ and $g$ are computed on the basis of $t$ observations at states $\{\bm{s}_1, \ldots, \bm{s}_t\}$.
Then, for both the reward and safety functions, the posterior mean, variance, and covariance are respectively represented as
\begin{alignat*}{3}
\bm{\mu}_t(\bm{s}) &=  \bm{k}_t^\top(\bm{s})(\bm{K}_t+\sigma^2 \bm{I})^{-1}\bm{y}_t, \\
\bm{\sigma}_t(\bm{s}) &= \bm{k}_t(\bm{s},\bm{s}), \\ \bm{k}_t(\bm{s},\bm{s}') &= \bm{k}(\bm{s},\bm{s}') - \bm{k}_t^\top(\bm{s})(\bm{K}_t+\sigma^2 \bm{I})^{-1}\bm{k}_t(\bm{s}'),
\end{alignat*}
where $\bm{k}_t(\bm{s}) = [k(\bm{s}_1,\bm{s}), \ldots, k(\bm{s}_t, \bm{s})]^\top$, and $\bm{K}_t$ is the positive definite kernel matrix.

\section{Background}

We define two kinds of predicted safe spaces inferred by a GP as in \citet{turchetta19goose}.
First, we consider a \textit{pessimistic} safe space, which contains states identified as safe with a greater probability than a pre-defined confidence level.
Second, we derive an \textit{optimistic} safe space that includes all states that may be safe with even a small probability.

\paragraph{Predicted pessimistic safe space.} 
We use the notion of a safe space in \citet{turchetta2016safe} as a \textit{predicted pessimistic safe space}.
For the probabilistic safety guarantee, two sets are defined.
The first set, $S^-_t$, simply contains the states that satisfy the safety constraint with high probability.
The second one, $\mathcal{X}^-_t$, additionally considers the ability to reach states in $S^-_t$ (i.e., reachability) and the ability to return to the previously identified safe set, $\mathcal{X}^-_{t-1}$ (i.e., returnability).
The algorithm probabilistically guarantees safety by allowing the agent to visit only states in $\mathcal{X}^-_t$.

Safety is evaluated in terms of the confidence interval inferred by the GP, which is represented as 
\[
Q_t(\bm{s}) := [\mu^g_{t-1}(\bm{s}) \pm \beta_t^{1/2} \sigma^g_{t-1}(\bm{s})],
\]
where $\beta_t \in \mathbb{R}$ is a scaling factor for the required level of safety.
We consider the intersection of $Q_t$ up to iteration $t$, which is defined as $C_t(\bm{s}) = Q_t(\bm{s}) \cap C_{t-1}(\bm{s})$, where $C_0(\bm{s}) = [h, \infty]$ for all $\bm{s} \in S_0$.
The lower and upper bounds on $C_t(\bm{s})$ are denoted by $l_t(\bm{s}):=\min C_t(\bm{s})$ and $u_t(\bm{s}):=\max C_t(\bm{s})$, respectively.

The first set $S^-_t$ contains states such that the safety constraint is satisfied with high probability. 
It is formulated using the lower bound of the safety function, $l$ and the Lipshitz constant, $L$, as follows:
\begin{equation*}
S^-_t = \{ \bm{s} \in \mathcal{S} \mid \exists \bm{s}' \in \mathcal{X}^-_{t-1}:
l_t(\bm{s}') - L \cdot d(\bm{s}, \bm{s}') \ge h \}.
\end{equation*}

Next, the reachable and returnable sets are considered.
Even though a state is in $S^-_t$, it might be surrounded by unsafe states.
Given a set $X$, the states reachable from $X$ in one step are given by
$R_{\text{reach}}(X) = X \cup \{\bm{s} \in \mathcal{S} \mid \exists \bm{s}' \in X, a \in \mathcal{A}: \bm{s}=f(\bm{s}', a)\}$.
Even after arriving at a state with reachability, the agent may \textit{not} be able to move to another state because of a lack of safe actions.
Hence, before moving to a state $\bm{s}$, we consider whether or not there is at least one viable path from $\bm{s}$.
The set of states from which the agent can return to a set $\bar{X}$ through another set of states $X$ in one step is given by $R_{\text{ret}}(X, \bar{X}) = \bar{X} \cup \{\bm{s} \in X \mid \exists a \in\mathcal{A}: f(\bm{s}, a) \in \bar{X} \}$. 
Thus, an $n$-step returnability operator is given by $R^n_\text{ret}(X, \bar{X}) = R_\text{ret}(X, R^{n-1}_\text{ret} (X, \bar{X})), \text{with } R^1_\text{ret}(X, \bar{X}) = R_\text{ret}(X, \bar{X})$.
Finally, the set containing all the states that can reach $\bar{X}$ along an arbitrary long path in $X$ is defined as $\bar{R}_{\text{ret}}(X,\bar{X}) = \lim_{n \rightarrow \infty} R^n_\text{ret}(X, \bar{X})$.

Finally, the desired \textit{pessimistic} safe space, $\mathcal{X}^-_t$ is a subset of $S^-_t$ and also satisfies the reachability and returnability constraints; that is,
\begin{equation*}
\mathcal{X}^-_t = \{ \bm{s} \in S^-_t \mid \bm{s} \in R_\text{reach}(\mathcal{X}^-_{t-1}) \cap \bar{R}_{\text{ret}}(S^-_t, \mathcal{X}^-_{t-1}) \}.
\end{equation*}

\paragraph{Predicted optimistic safe space.} 

As defined in \citet{wachi2018safe} and \citet{turchetta19goose}, an \textit{optimistic} safe space has rich information for inferring the safety function.
Let $\mathcal{X}^+_t$ denote the predicted optimistic safe space.
Similarly to $\mathcal{X}^-_t$, the \textit{optimistic} safe space, $\mathcal{X}^+_t$, is defined as
\begin{equation*}
\mathcal{X}^+_t = \{ \bm{s} \in S^+_t \mid \bm{s} \in R_\text{reach}(\mathcal{X}^+_{t-1}) \cap \bar{R}_{\text{ret}}(S^+_t, \mathcal{X}^+_{t-1}) \},
\end{equation*}
where $S^+_t$ is the set of states that may satisfy the safety constraint, which is written as
\[
S^+_t = \{ \bm{s} \in \mathcal{S} \mid \exists \bm{s}' \in \mathcal{X}^+_{t-1}:
u_t(\bm{s}') - L \cdot d(\bm{s}, \bm{s}') \ge h \}.
\]
Intuitively, $\mathcal{X}^+_t$ contains all states that may turn out to be safe even if the probability is low. 
In other words, $\mathcal{S} \setminus \mathcal{X}_t^+$ contains states that are unsafe with high probability.

\paragraph{Confidence interval.}

The correctness of $\mathcal{X}_t^+$ and $\mathcal{X}_t^-$ depends on the accuracy of the confidence interval inferred by the GP.
The \textit{conservativeness} can be tuned by using the parameter $\beta$, and the choice of this parameter was well-studied in \citet{srinivas2009gaussian} and \citet{chowdhury2017kernelized}. In the rest of this paper, we set the parameter to
\begin{alignat*}{2}
\beta_t = B^g + \sigma_g \sqrt{2(\Gamma^g_{t-1} + 1 + \log(1/\Delta^g))},
\end{alignat*}
where $B^g$ is a bound on the RKHS norm of $g$, $\Delta^g$ is the allowed failure probability, and the observation noise is $\sigma_g$-sub-Gaussian.
Also, $\Gamma^g$ quantifies the effective degrees of freedom associated with the kernel function, which represents the maximal mutual information that can be obtained about the GP prior.

Under the above definitions and assumptions, we have the following lemma regarding the correctness of the confidence intervals.

\begin{lemma}
\label{lemma:safety_confidence}
Assume that $\|g\|_k^2 \le B^g$ and $n^g_t \le \sigma_g$, $\forall t \ge 1$. If $\beta_t = B^g + \sigma_g \sqrt{2(\Gamma^g_{t-1} + 1 + \log(1/\Delta^g))}$, then
\[
| g(\bm{s}) - \mu_{t-1}^g(\bm{s}) | \le \beta^{1/2}_t \sigma_{t-1}^g(\bm{s})
\]
holds for all $t \ge 1$ with probability at least $1-\Delta^g$.
\end{lemma}

The above paradigm for the safety function can also be applied to the reward function.
Hence, we have a similar lemma for the reward function as well.

\begin{lemma}
\label{lemma:reward_confidence}
Assume that $\|r\|_k^2 \le B^r$ and $n^r_t \le \sigma_r$, $\forall t \ge 1$. If $\alpha_t = B^r + \sigma_r \sqrt{2(\Gamma^r_{t-1} + 1 + \log(1/\Delta^r))}$, then
\[
| r(\bm{s}) - \mu_{t-1}^r(\bm{s}) | \le \alpha^{1/2}_t \sigma_{t-1}^r(\bm{s})
\]
holds for all $t \ge 1$ with probability at least $1-\Delta^r$.
\end{lemma}
These lemmas follow from Theorem~2 in \citet{chowdhury2017kernelized}.

\paragraph{Optimal solution.}
Here, we define the optimal policy in our problem setting.
Under the optimal policy, $\pi^*$, the value function, $V_{\mathcal{M}}$, satisfies the following Bellman equation:
\begin{alignat*}{2}
V_{\mathcal{M}}^*(\bm{s}_t) = \max_{\bm{s}_{t+1} \in \bar{R}_{\epsilon_g}(S_0)} \left[\ r(\bm{s}_{t+1}) + \gamma  V_{\mathcal{M}}^*(\bm{s}_{t+1}) \ \right],
\end{alignat*}
where $\bar{R}_{\epsilon_g}(S_0)$ is the largest set that can be safely learned up to $\epsilon_g$ accuracy (for a formal definition, see Appendix A or \citet{turchetta2016safe}).
In our problem setting, in which the reward and safety functions are unknown a priori, the above Bellman equation cannot be solved directly.
Our ultimate objective is to obtain a policy whose value is close to $\mathcal{V}^*_\mathcal{M}$ while guaranteeing satisfaction of the safety constraint.

\section{Algorithm}

We now introduce our proposed algorithm, \algo, for achieving a near-optimal policy with respect to the cumulative reward while guaranteeing safety.

We first give an overview of \algo, which is outlined as Algorithm~\ref{algorithm1}.
We extend a stepwise approach in \citet{sui2018stagewise} from state-less to stateful settings.
Basically, our algorithm consists of two steps.
In the first step, the agent expands the pessimistic safe region while guaranteeing safety (lines 2$-$17).
Next, it explores and exploits the reward in the safe region certified in the first step (lines 18$-$23).
The reason for this stepwise approach is that we can neglect uncertainty related to the a priori unknown safety function once the safe region is fixed.

However, a pure stepwise approach does not stop exploring the safe region until the convergence of the GP confidence interval (lines 15$-$16).
This formulation often requires the agent to execute a great number of actions for exploring safety.
Hence, to achieve near-optimality while executing a smaller number of actions, we also propose the \es$^2$ algorithm.\footnote{Both \es$^2$ and \pes$^2$ do \textit{not} affect the agent's safety.}
This algorithm checks whether the current safe region is sufficient for achieving near-optimality (lines 12$-$14), which maintains the theoretical guarantee with respect to both the satisfaction of the safety constraint and the near-optimality of the cumulative reward.
We further propose a practical \es$^2$ algorithm, called \pes$^2$, with better empirical performance, although it does not provide a theoretical guarantee in terms of the near-optimality of the cumulative reward.

\begin{algorithm}[t]
\caption{\textbf{\space \algo\ with \es$^2$}}
\begin{small}
\label{algorithm1}
\textbf{Input}: states $\mathcal{S}$, actions $\mathcal{A}$, transition function $f$, kernel $k^r$ for reward, kernel $k^g$ for safety, GP prior for reward, GP prior for safety, safety threshold $h$, discount factor $\gamma$, Lipschitz constant $L$, initial safe space $S_0$. \par
\begin{algorithmic}[1]
\STATE $C_0(\bm{s}) \leftarrow [h, \infty)$ for all $\bm{s} \in S_0$
\STATE // Exploration of safety
\LOOP
\STATE $S^-_t \leftarrow \{ \bm{s} \in \mathcal{S} \mid \exists \bm{s}' \in \mathcal{X}^-_{t-1}:
l_t(\bm{s}') - L \cdot d(\bm{s}, \bm{s}') \ge h \}$
\STATE $S^+_t \leftarrow \{ \bm{s} \in \mathcal{S} \mid \exists \bm{s}' \in \mathcal{X}^+_{t-1}:
u_t(\bm{s}') - L \cdot d(\bm{s}, \bm{s}') \ge h \}$
\STATE $\mathcal{X}^-_t \leftarrow \{ \bm{s} \in S^-_t \mid \bm{s} \in R_\text{reach}(\mathcal{X}^-_{t-1}) \cap \bar{R}_{\text{ret}}(S^-_t, \mathcal{X}^-_{t-1}) \}$
\STATE $\mathcal{X}^+_t \leftarrow \{ \bm{s} \in S^+_t \mid \bm{s} \in R_\text{reach}(\mathcal{X}^+_{t-1}) \cap \bar{R}_{\text{ret}}(S^+_t, \mathcal{X}^+_{t-1}) \}$
\STATE $G_t \leftarrow \{ s \in \mathcal{X}_t^- \mid e_t(s) > 0 \}$
\STATE $\xi \leftarrow \argmax_{\bm{s} \in G_t} w_t(\bm{s})$
\STATE Update GPs for both reward and safety on way to $\xi$
\STATE $t \leftarrow t + T_{\bm{s}_{t-1} \rightarrow \xi}$ and $\bm{s}_t \leftarrow \xi$
\STATE // \es$^2$ algorithm
\STATE $\mathcal{Y}_t \leftarrow \{\bm{s}' \in \mathcal{S}^+ \mid \forall \bm{s} \in \mathcal{X}^-_t: \bm{s}' = f(\bm{s}, \pi^*_y(a \mid \bm{s})) \}$
\STATE \textbf{if} $\mathcal{Y}_t \subseteq \mathcal{X}_t^-$ \textbf{then} \textbf{break}
\STATE // Typical termination condition
\STATE \textbf{if} $\max_{\bm{s} \in G_t} w_t(s) < \epsilon_g$ \textbf{then} \textbf{break}
\ENDLOOP
\STATE // Exploration and exploitation of reward
\LOOP
\STATE $U_t \leftarrow \mu^r_t + \alpha_{t+1} \cdot \sigma^r_t$
\STATE $J_{\mathcal{Y}}^*(\bm{s}_t) \leftarrow \max_{s_{t+1} \in \mathcal{Y}_t} \bigl[U_t(\bm{s}_{t+1}) + \gamma J_{\mathcal{Y}}^*(\bm{s}_{t+1}) \bigr]$
\STATE $\bm{s}_{t+1} \leftarrow \argmax_{s_{t+1} \in \mathcal{Y}_t} J_{\mathcal{Y}}^*(\bm{s}_t)$
\ENDLOOP
\end{algorithmic}
\end{small}
\end{algorithm}

\subsection{Exploration of Safety (Step 1)}

First, we consider how to explore the safety function.
As a scheme to expand the safe region, we consider ``expanders'' as in \citet{sui2015safe} and \citet{turchetta2016safe}.
Expanders are states that may expand the predicted safe region, which is defined as $G_t = \{ s \in \mathcal{X}_t^- \mid e_t(s) > 0 \}$, where $e_t(s) = |s' \in \mathcal{S} \setminus S_t^- \mid u_t(s) - Ld(s, s') \ge h|$.

The \textit{efficiency} of expanding the safe region is measured by the width of the safety function's confidence interval, defined as 
\[
w_t(\bm{s}) = u_t(s) - l_t(s).
\]
The agent safely and efficiently expands the safe region by sampling the state with the maximum value of $w$ among the expanders, $G_t$.
Hence, the agent sets the temporal goal according to
\[
\xi = \argmax_{\bm{s} \in G_t} w_t(\bm{s}).
\]
Then, within the predicted safe space $\mathcal{X}_t^-$, it chooses a path to get to $\xi$ from the current state $\bm{s}_{t-1}$ so as to minimize the cost (e.g., the path length).
In our experiment, we simply minimized the path length. 
By defining the cost as related to $w$ (e.g., $1/w$), however, the agent could explore safety more actively on the way to $\xi$.

The previous work \cite{sui2015safe,turchetta2016safe,sui2018stagewise} terminated safety exploration when the desired accuracy was achieved for every state in $G_t$; that is,
\begin{equation}
\label{eq:terminal_cnd_prev}
    \max_{s \in G_t} w_t(\bm{s}) \le \epsilon_g.
\end{equation}
Unfortunately, this termination condition often requires a great number of iterations.
For the purpose of maximizing the cumulative reward, it often leads to the loss of reward. 
Therefore, in Section~\ref{ES2}, we propose the \es$^2$ algorithm to improve this point.

\subsection{Exploration and Exploitation of Reward (Step 2)}

Once expansion of the safe region is completed, the agent guarantees safety as long as it is in $\mathcal{X}^-$ and does not have to expand the safe region anymore. 
Hence, all we have to do is optimize the cumulative reward in $\mathcal{X}^-$. As such, a simple approach is to follow the \textit{optimism in the face of uncertainty} principle as in \citet{strehl2008analysis} and \citet{auer2007logarithmic}, then to consider the ``exploration bonus'' represented by R-MAX \citep{brafman2002r} and Bayesian Exploration Bonus (BEB, \citet{kolter2009near}).

Specifically, in accordance with Lemma~\ref{lemma:reward_confidence}, we optimize the policy by \textit{optimistically} measuring the reward with the (probabilistic) upper confidence bound,
\[
U_t(\bm{s}) := \mu_t^r(\bm{s})  + \alpha^{1/2}_{t+1} \cdot \sigma_t^r(\bm{s}).
\]
In this reward setting, the second term on the right-hand side corresponds to the exploration bonus.
For balancing the exploration and exploitation in terms of reward, we solve the following Bellman equation: 
\begin{alignat*}{2}
\nonumber
J_{\mathcal{X}}^*(\bm{s}_t, \bm{b}^r_t, \bm{b}^g_t) = \!\!\!\max_{s_{t+1} \in \mathcal{X}_{t^*}^-} \bigl[ U_t(\bm{s}_{t+1}) + \gamma J_{\mathcal{X}}^*(\bm{s}_{t+1}, \bm{b}^r_t, \bm{b}^g_t) \bigr],
\end{alignat*}
where $\bm{b}^r=(\mu^r, \sigma^r)$ and $\bm{b}^g=(\mu^g, \sigma^g)$ are the beliefs over reward and safety, respectively.
Also, $t^*$ is the time step when the termination condition (\ref{eq:terminal_cnd_prev}) is satisfied.
Note that $\bm{b}^r$ and $\bm{b}^g$ are not updated; hence, we can solve the above equation with standard algorithms (e.g., value iteration).

\subsection{Early Stopping of Exploration of Safety (\es$^2$)}
\label{ES2}

\begin{figure}[t]
    \centering
    \includegraphics[width=75mm]{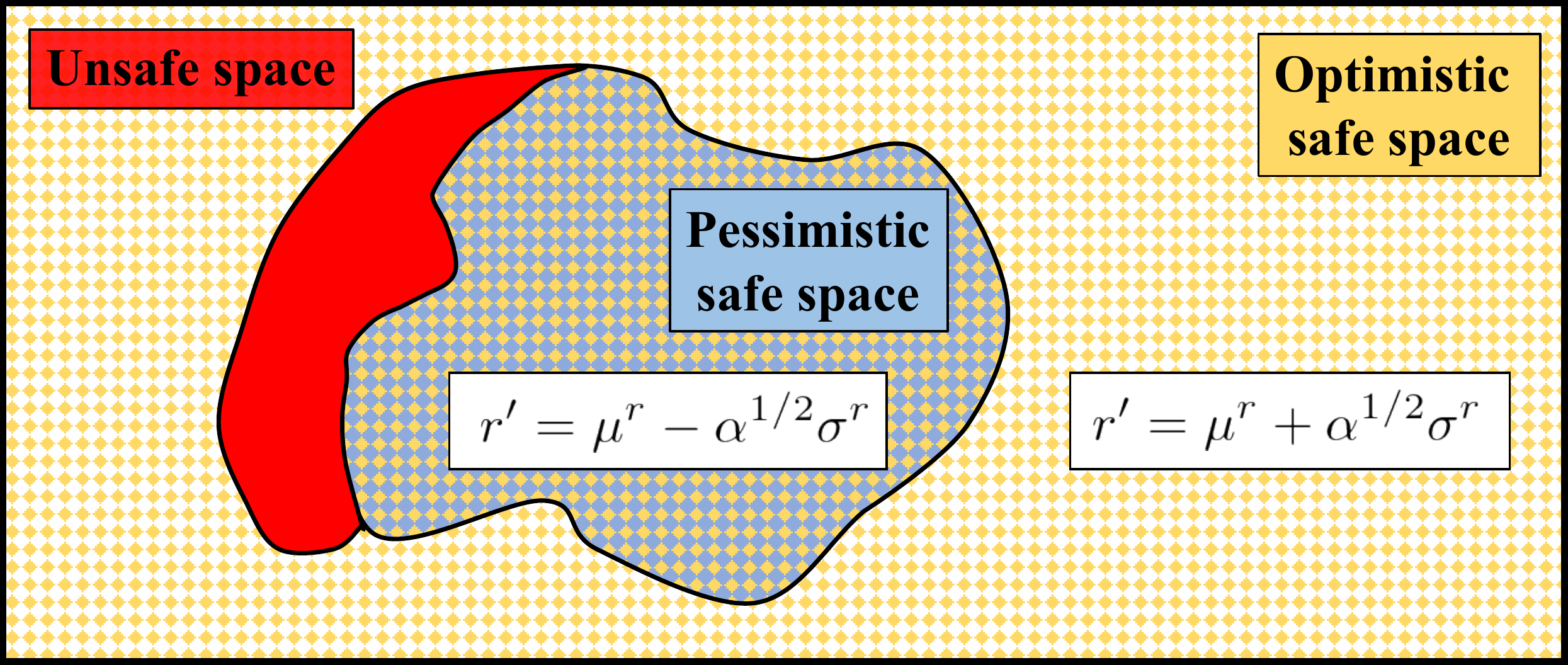}
    \caption{Illustration of $\mathcal{M}_y$, used in the \es$^2$ algorithm. The yellow and blue regions represent $\mathcal{X}_t^+$ and $\mathcal{X}_t^-$, respectively. The red region (i.e., $\mathcal{X} \setminus \mathcal{X}_t^+$) is unsafe with high probability.}
    \label{fig:es2}
\end{figure}

We have proposed a stepwise approach for exploring and optimizing the constrained MDP.
In the first step when the safe region is expanded, however, the existing safe exploration algorithms \citep{sui2015safe,turchetta2016safe,sui2018stagewise} continue exploring the state space until convergence of the confidence interval, $w$, which generally leads to a large number of iterations.
Our primary objective is to maximize the cumulative reward; hence, we should stop exploring safety if further exploration will not lead to maximizing the cumulative reward.

While exploring the safety function, we check whether the step can be migrated.
As such, we consider the following additional MDP,
\begin{alignat*}{2}
\mathcal{M}_y = \langle \mathcal{X}^+, \mathcal{A}, f, r', g, \gamma \rangle.
\end{alignat*}
As shown in Figure~\ref{fig:es2}, the differences from the original MDP, $\mathcal{M}$, lie in the state space and the reward function.
The state space of $\mathcal{M}_y$ is defined as the optimistic safe space (i.e., $\mathcal{X}^+$), while the reward function is defined as follows:
\begin{eqnarray}
r' :=
  \left\{\begin{array}{ll}
  \mu^r + \alpha^{1/2} \sigma^r & \quad \text{if}\ \  \bm{s} \in \mathcal{X}^+_t \setminus \mathcal{X}^-_t, \\
  \mu^r - \alpha^{1/2} \sigma^r & \quad \text{if}\ \  \bm{s} \in \mathcal{X}^-_t.
  \end{array}
  \right.
  \label{eqn:reward_new_MDP}
\end{eqnarray}
In the pessimistic safe space, the reward is defined as the lower bound; otherwise, it is defined as the upper bound.
This definition of the reward function encourages the agent to explore outside the predicted safe space, $\mathcal{X}_t^-$.
Using the new MDP above, we consider the set of states that the agent will visit at the next time step, defined as
\[
\mathcal{Y}_t = \{\bm{s}' \in \mathcal{S}^+ \mid \forall \bm{s} \in \mathcal{X}^-_t: \bm{s}' = f(\bm{s}, \pi_y^*(a \mid \bm{s})) \},
\]
where $\pi^*_y$ is the optimal policy for $\mathcal{M}_y$, obtained by maximizing the following value function:
\begin{equation}
\label{eq:V_es2}
V_{\mathcal{M}_y}(\bm{s}_t) = \max_{s_{t+1} \in \mathcal{X}_t^+} [\ r'(\bm{s}_{t+1}) + \gamma V_{\mathcal{M}_y}(\bm{s}_{t+1}) \ ].
\end{equation}
Finally, we stop exploring the safety function if the following equation holds:
\begin{equation}
\mathcal{Y}_t \subseteq \mathcal{X}^-_t.
\label{eq:stopping_cond}
\end{equation}
Intuitively, we stop expanding the safe space if the direction of the optimal policy for $\mathcal{M}_y$ heads for the inside of $\mathcal{X}_t^-$.
If the agent tries to stay in $\mathcal{X}_t^-$ even under the condition that the reward is defined as in (\ref{eqn:reward_new_MDP}), then we do not have to expand the safe region anymore.

When the \es$^2$ algorithm confirms satisfaction of the above condition, we move on to the next step and then optimize the cumulative reward in $\mathcal{Y}_t$; that is,
\begin{alignat*}{2}
\nonumber
J_{\mathcal{Y}}^*(\bm{s}_t, \bm{b}^r_t, \bm{b}^g_t) = \max_{s_{t+1} \in \mathcal{Y}_t} \bigl[U_t(\bm{s}_{t+1}) + \gamma J_{\mathcal{Y}}^*(\bm{s}_{t+1}, \bm{b}^r_t, \bm{b}^g_t) \bigr].
\label{eq:J_Y}
\end{alignat*}

\subsection{More Practical \es$^2$ Algorithm (\pes$^2$)}

As we will prove in Section~\ref{sec:theory}, the \es$^2$ algorithm provides us with a theoretical guarantee with respect to the cumulative reward.
Unfortunately, this guarantee is achieved at the expense of empirical performance.
The issue with the pure \es$^2$ algorithm lies in the state constraint in (\ref{eq:V_es2}); that is, the value function is calculated under the assumption that all the states in the \textit{optimistic} safe space, $\mathcal{X}^+$, will be identified as safe.
This assumption is necessary for the theoretical guarantee, but, in practice, it would be more reasonable to measure the probability of a state being identified as safe.
Because the safety function is inferred as a Gaussian distribution for each state, an example of such a probability is a complementary error function; that is, we define the following probability,
\begin{alignat*}{2}
p(\bm{s}, \bm{b}^g) = \Pr \left[ g(\bm{s}) \ge h \mid \bm{b}^g \right] 
\approx 1  - \frac{1}{2} \mathrm{erfc} \left(\frac{\mu^g(\bm{s}) - h}{\sqrt{2} \sigma^g(\bm{s})} \right).
\end{alignat*}

Here, we introduce a new \textit{virtual} state, $\bm{z}$.
Concretely, for $\bm{z}$, the reward and transition probability are defined as $r(\bm{z}) = 0$ and $P(\bm{z} \mid \bm{z}, a, \bm{b}^g) = 1$ for all $a$ and $\bm{b}^g$, respectively.
Hence, using $\bm{z}$, we define a virtual transition probability $P^z_{\bm{x}} = \Pr[\bm{x} \mid \bm{s}_t, a_t, \bm{b}_t^g]$ as follows:
\begin{eqnarray*}
P^z_{\bm{x}} :=
  \left\{\begin{array}{ll}
  p(\bm{s}_{t+1}, \bm{b}_t^g) & \ \text{if}\ \ \bm{x} = \bm{s}_{t+1}, \\
  1 - p(\bm{s}_{t+1}, \bm{b}_t^g) & \ \text{if}\ \ \bm{x} = \bm{z}.
  \end{array}
  \right.
  \label{eqn:virtual_transition}
\end{eqnarray*}

\begin{figure}[t]
\centering
\includegraphics[width=75mm]{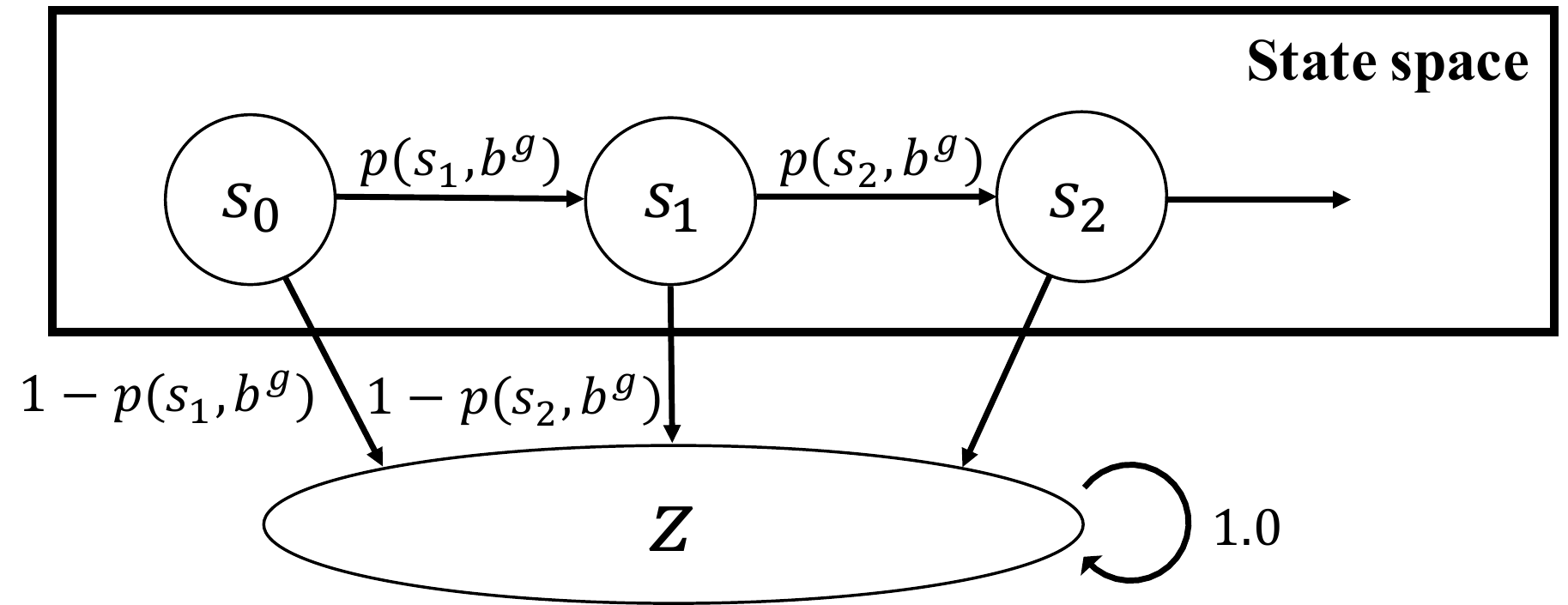}
\caption{Illustration of $\mathcal{M}_z$. This MDP is characterized by the virtual state $z$ and the virtual transition probability $P^z$.}
\label{fig:mdp_w_z}
\end{figure}

Hence, by introducing $z$ and $P^z$, we define the following MDP with the smooth, continuous transition probability:
\begin{alignat*}{2}
\mathcal{M}_z = \langle \mathcal{X}^+ \cup \{z\}, \mathcal{A}, P^z, r', g, \gamma \rangle.
\end{alignat*}
Figure~\ref{fig:mdp_w_z} shows a conceptual image of this MDP. 
Intuitively, in $\mathcal{M}_z$, the agent optimizes the policy under the virtual condition that a state-action pair may lead to the extremely undesirable state, $z$, with probability $1-p$.
Then, we solve the following equation instead of solving (\ref{eq:V_es2}):
\begin{equation*}
\label{eq:V_es2_practical}
V_{\mathcal{M}_z}(\bm{s}_t) = \max_{s_{t+1} \in \mathcal{X}_t^+} [ P^z_{\bm{s}_{t+1}} \cdot \{ r'(\bm{s}_{t+1}) + \gamma V_{\mathcal{M}_z}(\bm{s}_{t+1}) \} ].
\end{equation*}
For this equation, we used $r(\bm{z})=0$ and $V(\bm{z})=0$. 
For the optimal policy $\pi^*_z$ obtained by solving the above equation, we stop exploring the safety function if the following equation holds:
\[
\mathcal{Z}_t := \{ \bm{s}' \in \mathcal{S}^+ \mid \forall \bm{s} \in \mathcal{X}_t^-: \bm{s}' = f(\bm{s}, \pi^*_z(a \! \mid \! \bm{s})) \} \subseteq \mathcal{X}_t^-.
\]
Then, we optimize the cumulative reward in $\mathcal{Z}_t$ by solving the following equation:
\begin{alignat*}{2}
\nonumber
J_{\mathcal{Z}}^*(\bm{s}_t, \bm{b}^r_t, \bm{b}^g_t) = \max_{s_{t+1} \in \mathcal{Z}_t} \bigl[U_t(\bm{s}_{t+1}) + \gamma J_{\mathcal{Z}}^*(\bm{s}_{t+1}, \bm{b}^r_t, \bm{b}^g_t) \bigr].
\label{eq:J_Z}
\end{alignat*}

\section{Theoretical Results}
\label{sec:theory}

We now provide theoretical guarantees on the safety and near-optimality of our proposed algorithm.
Theorem~\ref{theo:SafeGuarantee} is associated with the safe expansion stage (i.e., step 1), which guarantees safety and convergence to the safe region.
Theorem~\ref{theo:NearOpt} ensures convergence toward the near-optimal cumulative reward.
Theorem~\ref{theo:es2} ensures that \algo\ still achieves the near-optimal cumulative reward even when the \es$^2$ algorithm is used.
In the rest of this paper, let $V_{\max} = R_{\max}/(1-\gamma)$. Also, let $D: \mathcal{M} \rightarrow \mathbb{R}$ be a diameter of an MDP, defined as $D(\mathcal{M}) = \min_\pi \max_{s^1 \ne s^2} T_{s^1 \rightarrow s^2}^\pi$, where $T_{s^1 \rightarrow s^2}^\pi$ is the expected number of time steps that policy $\pi$ takes to move from $s^1$ to $s^2$.

\subsection{Safety Guarantee and Completeness}

We first present a theorem related to the safety guarantee and completeness.

\begin{theorem}
\label{theo:SafeGuarantee}
Assume that the safety function $g$ satisfies $\|g\|_k^2 \le B^g$ and is $L$-Lipschitz continuous.
Also, assume that $S_0 \ne \emptyset$ and $g(\bm{s}) \ge h$ for all $\bm{s} \in S_0$.
Fix any $\epsilon_g > 0$ and $\Delta^g \in (0, 1)$.
Suppose that we conduct the stage of ``exploration of safety'' with the noise $n_t^g$ being $\sigma_g$-sub-Gaussian, and that $\beta_t = B^g + \sigma_g \sqrt{2(\Gamma^g_{t-1} + 1 + \log(1/\Delta^g))}$ until $\max_{\bm{s} \in G_t} w_t(\bm{s}) < \epsilon_g$ is achieved.
Finally, let $t^*$ be the smallest integer satisfying \[
\frac{t^*}{\beta_{t^*}\Gamma^g_{t^*}} \ge \frac{C_g|\bar{R}_0(S_0)|}{\epsilon_g^2} D(\mathcal{M}),
\]
with $C_g=8/\log(1+\sigma_g^{-2})$.
Then, the following statements jointly hold with probability at least $1 - \Delta^g$: 
\begin{itemize}
    \item $\forall t \ge 1$, $g(\bm{s}_t) \ge h$,
    \item $\exists t_0 \le t^*$, $\bar{R}_{\epsilon_g} (S_0) \subseteq \mathcal{X}^-_{t_0} \subseteq \bar{R}_0 (S_0)$.
\end{itemize}
\end{theorem}

A proof is presented in the supplemental material. Theorem~\ref{theo:SafeGuarantee} guarantees that \algo\ is safe in the stage of \textit{exploration of safety} (i.e., step 1), as well as in the stage of \textit{optimization of reward} (i.e., step 2), with high probability.
In addition, after a sufficiently large number of time steps, $\mathcal{X}^-$ is guaranteed to be a super-set of $\bar{R}_{\epsilon_g} (S_0)$.

\subsection{Near-Optimality}

We next present a theorem on the near-optimality with respect to the cumulative reward.

\begin{theorem}
\label{theo:NearOpt}
Assume that the reward function $r$ satisfies $\|r\|_k^2 \le B^r$, and that the noise is $\sigma_r$-sub-Gaussian. 
Let $\pi_t$ denote the policy followed by \algo\ at time $t$, and let $\bm{s}_t$ and $\bm{b}^r_t, \bm{b}^g_t$ be the corresponding state and beliefs, respectively.
Let $t^*$ be the smallest integer satisfying
$\frac{t^*}{\beta_{t^*}\Gamma^g_{t^*}} \ge \frac{C_g|\bar{R}_0(S_0)|}{\epsilon_g^2} D(\mathcal{M})$, and fix any $\Delta^r \in (0, 1)$.
Finally, set $\alpha_t = B^r + \sigma_r \sqrt{2(\Gamma^r_{t-1} + 1 + \log(1/\Delta^r))}$ and
\[
\epsilon^*_V = V_{\max} \cdot (\Delta^g + \Sigma_{t^*}^r/R_{\max}),
\]
with $\Sigma^r_{t^*} =  \frac{1}{2}\sqrt{\frac{C_r \alpha_{t^*} \Gamma_{t^*}^r}{t^*}}$.
Then, with high probability,
\[
V^{\pi_t}(\bm{s}_t, \bm{b}^r_t, \bm{b}^g_t) \ge V^*(\bm{s}_t) - \epsilon^*_V
\]
--- i.e., the algorithm is $\epsilon^*_V$-close to the optimal policy --- for all but $t^*$ time steps, while guaranteeing safety with probability at least $1-\Delta^g$.
\end{theorem}

A detailed proof of Theorem~\ref{theo:NearOpt} is presented in the supplemental material.
The proof is based on the following idea.
After the agent fully explores the safe space, $\mathcal{X}^-$ satisfies $\bar{R}_{\epsilon_g} (S_0) \le \mathcal{X}^- \le \bar{R}_0 (S_0)$, and states in $\mathcal{X}^-$ are safe with high probability.
Once $\mathcal{X}^-$ converges, the probability of leaving the ``known'' safe space is small; hence, Theorem~\ref{theo:NearOpt} follows by adapting standard arguments from previous PAC-MDP results.
The key condition that allows us to prove the near-optimality of \algo\ is that, at every time step, the agent is optimistic with respect to the reward, and this optimism decays given a sufficient number of samples.
By optimizing the cumulative reward in $\mathcal{X}^-$ according to the optimism in the face of uncertainty principle, the acquired policy is $\epsilon^*_V$-close to the optimal policy in the original safety-constrained MDP.

Finally, we present a theoretical result related to the \es$^2$ algorithm. Specifically, we prove that \es$^2$ maintains the near-optimality of \algo.

\begin{theorem}
\label{theo:es2}
Assume that the reward function $r$ satisfies $\|r\|_k^2 \le B^r$, and that the noise is $\sigma_r$-sub-Gaussian.
Let $\pi_t$ denote the policy followed by \algo\ with the \es$^2$ algorithm at time $t$, and let $\bm{s}_t$ and $\bm{b}^r_t, \bm{b}^g_t$ be the corresponding state and beliefs, respectively.
Let $\tilde{t}$ be the smallest integer for which (\ref{eq:stopping_cond}) holds, and fix any $\Delta^r \in (0, 1)$.
Finally, set $\alpha_t = B^r + \sigma_r \sqrt{2(\Gamma^r_{t-1} + 1 + \log(1/\Delta^r))}$ and
\[
\tilde{\epsilon}_V = V_{\max} \cdot (\Delta^g + \Sigma_{\tilde{t}}^r/R_{\max}),
\]
with $\Sigma^r_{\tilde{t}} = \frac{1}{2} \sqrt{\frac{C_r \alpha_{\tilde{t}} \Gamma_{\tilde{t}}^r}{\tilde{t}}}$.
Then, with high probability,
\[
V^{\pi_t}(\bm{s}_t, \bm{b}^r_t, \bm{b}^g_t) \ge V^*(\bm{s}_t) - \tilde{\epsilon}_V
\]
--- i.e., the algorithm is $\tilde{\epsilon}_V$-close to the optimal policy --- for all but $\tilde{t}$ time steps while guaranteeing safety with probability at least $1-\Delta^g$.
\end{theorem}
The proof of Theorem~\ref{theo:es2} is presented in the supplemental material.
The proof is based on the following idea.
When the condition in (\ref{eq:stopping_cond}) is satisfied, the agent will not leave~$\mathcal{Y}$, and a near-optimal policy is obtained by optimizing the cumulative reward only in $\mathcal{Y}$ with the optimistically measured reward.
Also, as long as the agent is in $\mathcal{Y}$ $(\subseteq \mathcal{X}^-)$, safety is guaranteed with high probability.
The proof is similar to that for Theorem~\ref{theo:NearOpt}.

\section{Experiment}

In this section, we evaluate the performance of \algo\ through two experiments. 
One used a synthetic environment, while the other simulated Mars surface exploration.
We also show the effectiveness of our \es$^2$ and \pes$^2$ algorithms.

\subsection{Synthetic \textbf{\gpsafetygym}\ Environment}

\begin{figure*}[t]
\centering
    \hspace{3mm}
	\subfigure[\gpsafetygym.]{%
		\includegraphics[clip, height=34mm]{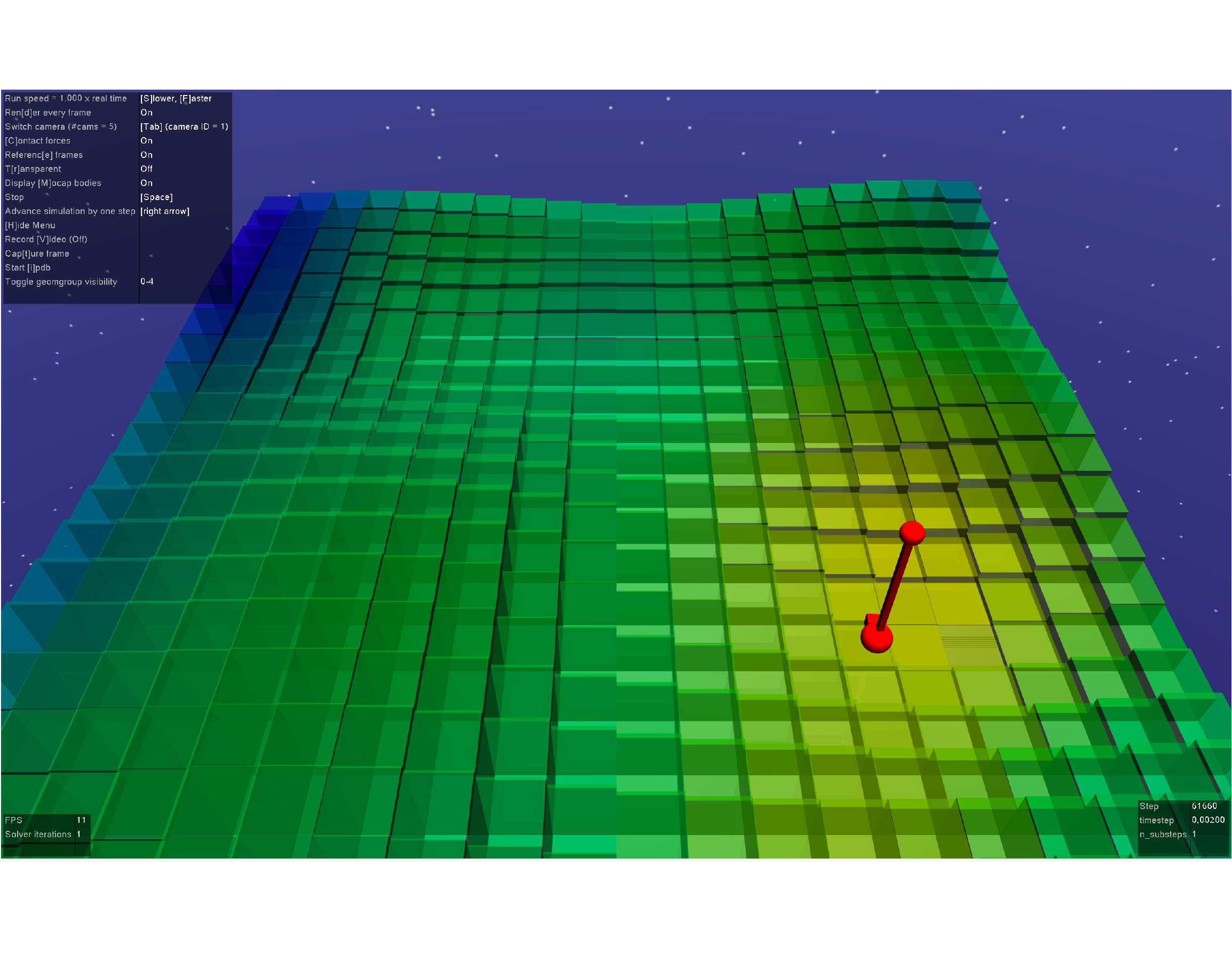}}%
	\hspace{4mm}
	\subfigure[Performance comparison.]{%
		\includegraphics[clip, height=34mm]{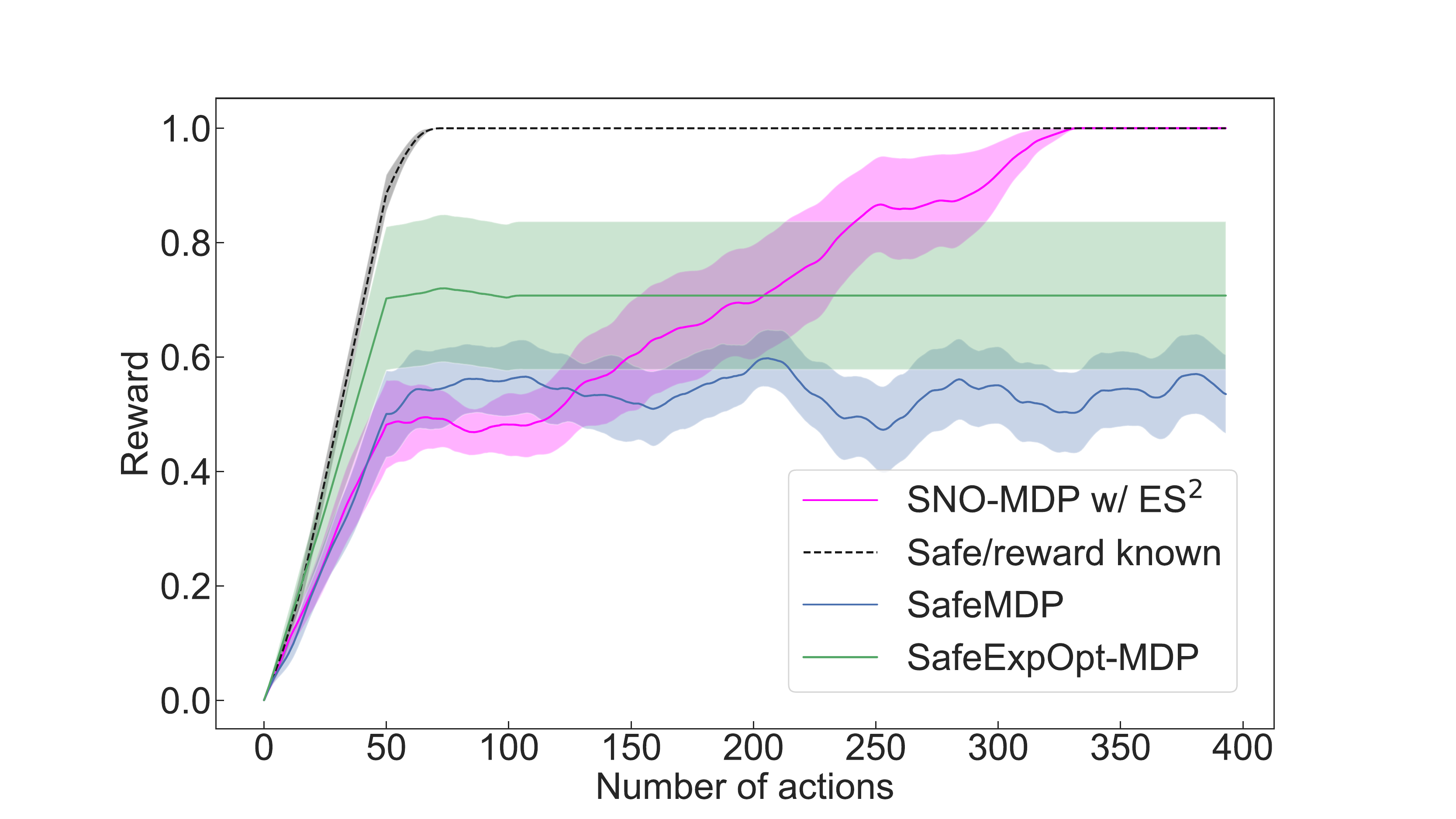}}%
	\hspace{-4mm}
	\subfigure[Effects of \es$^2$ and \pes$^2$.]{%
		\includegraphics[clip, height=34mm]{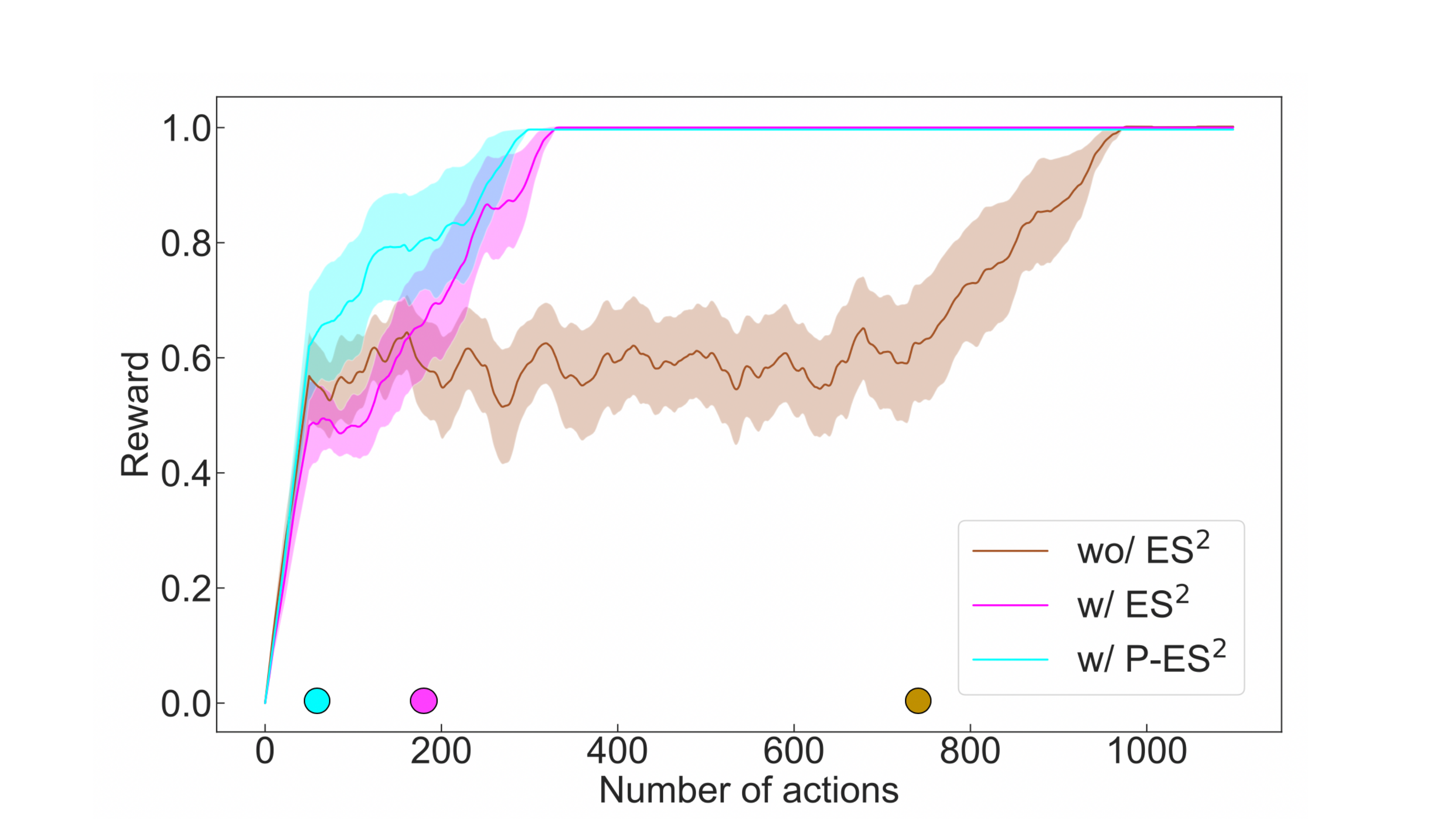}}%
	\caption{Experiment with synthetic data. (a) Example screen capture from the \gpsafetygym\ environment. (b) Average reward over the episodes, comparing the performance of \algo\ with \es$^2$ and the baselines. (c) Average reward over the episodes, showing the effects of \es$^2$ and \pes$^2$. The colored circles represent when the transition from safe exploration to reward optimization happens for each method. In both (b) and (c), the reward is normalized with respect to the \textsc{Safe/Reward known} case.}
	\label{fig:syn}
\end{figure*}

\paragraph{Settings.}

We constructed a new open-source environment for safe RL simulations named \gpsafetygym. This environment was built based on OpenAI Safety-Gym \cite{Ray2019}.
As shown in Figure~\ref{fig:syn}(a), \gpsafetygym\ represents the reward by a color (yellow: high; green: medium; blue: low), and the safety by height.

We considered a $20 \times 20$ square grid in which the reward and safety functions were randomly generated.
At every time step, an agent chose an action from \textit{stay}, \textit{up}, \textit{right}, \textit{down}, and \textit{left}.
The agent predicted the reward and safety functions by using different kernels on the basis of previous observations.
In this simulation, we allowed the agent to observe the reward and safety function values of the current state and neighboring states.
The kernel for reward was a radial basis function (RBF) with the length-scales of $2$ and prior variance of $1$.
The kernel for safety was also an RBF with the length-scales of $2$ and prior variance of $1$.
Finally, we set the discount factor to $\gamma = 0.99$, and confidence intervals parameters to $\alpha_t = 3$ and $\beta_t = 2$ for all $t \ge 1$.

\paragraph{Baselines.}

We empirically compared the performance of our \algo\ with \textsc{SafeMDP} \citep{turchetta2016safe} and \textsc{SafeExpOpt-MDP} \citep{wachi2018safe}, as well as a case called \textsc{Safe/Reward known}. 
In \textsc{SafeMDP}, the agent tries to expand the safe region without considering the reward.
In \textsc{SafeExpOpt-MDP}, the agent attempts to maximize the cumulative reward while leveraging the difference between the value function in $\mathcal{X}_t^+$ and that in $\mathcal{X}_t^-$.
Finally, \textsc{Safe/Reward known} is a non-exploratory case in which the  safety and reward functions are known a priori.

\paragraph{Metrics.}
We used the cumulative reward and the number of unsafe actions as comparison metrics.

\paragraph{Results.}

Figure~\ref{fig:syn}(b) compares the performance of \algo\ and the baselines in terms of the reward. 
For these results, the average reward was measured over the previous 50 time steps. 
\algo\ achieved the optimal reward after shifting to the stage of reward optimization, which outperforms \textsc{SafeMDP} and \textsc{SafeExpOpt-MDP} in terms of the reward after a sufficiently large number of time steps.
The \textsc{SafeMDP} agent did not aim to maximize the cumulative reward, and the \textsc{SafeExpOpt-MDP} agent was sometimes stucked in a local optimum when the expansion of the safe region was insufficient. 
Figure~\ref{fig:syn}(c) shows the empirical performance of the \es$^2$ and \pes$^2$ algorithms. 
\pes$^2$ achieved faster convergence in terms of the reward than the original \es$^2$ did. 
Also, all methods, including the baselines, did \textit{not} take any unsafe actions.

\subsection{Simulated Mars Surface Exploration}

\paragraph{Settings.}
We then conducted an experiment based on a Mars surface exploration scenario, as in \citet{turchetta2016safe} and \citet{wachi2018safe}. 
In this simulation, we used a publicly available Mars digital elevation model (DEM) that was created from observation data captured by the high-resolution imaging science experiment (HiRISE) camera \cite{mcewen2007mars}.

\begin{wrapfigure}{r}{41mm}
\includegraphics[width=39mm]{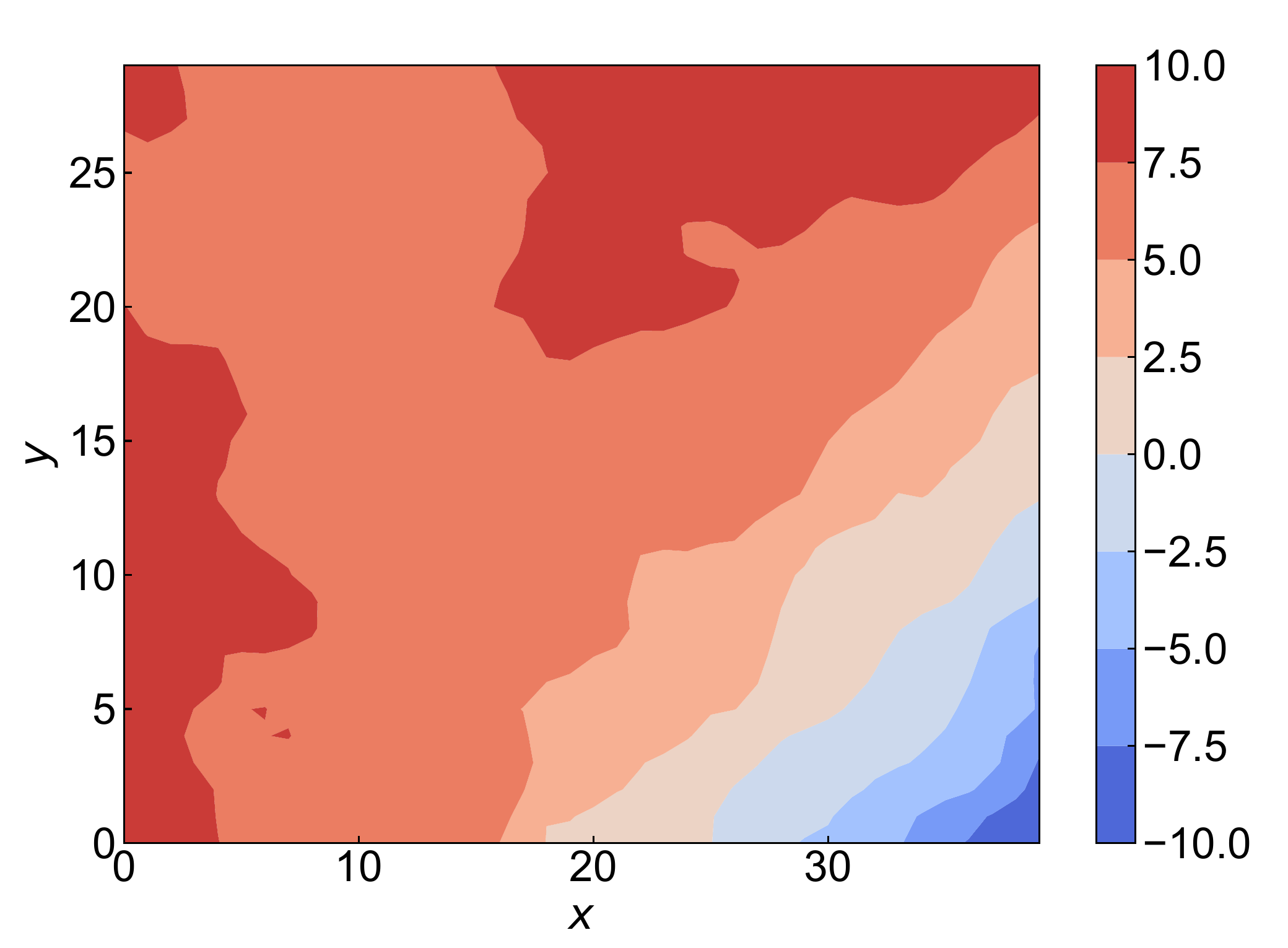}
\caption{Mars terrain data.}
\label{fig:mars}
\end{wrapfigure}
We created a 40 $\times$ 30 rectangular grid-world by clipping a region around latitude 30$^\circ$6' south and longitude 202$^\circ$2' east, as shown in Figure~\ref{fig:mars}.
At every time step, the rover took one of five actions: \textit{stay, up, down, left, and right}.
We assumed that any state in which the slope angle was greater than $25^\circ$ were unsafe.
The safety function $g$ was defined as the slope angle calculated from the DEM, and the safety threshold was $h = - \tan(25^\circ)$.

The rover predicted the elevation by using a GP with a Mat\'{e}rn kernel with $\nu= 5/2$. 
The length-scales were $15$~m, and the prior variance over elevation was $100$~m$^2$. 
We assumed a noise standard deviation of $0.075$ m. 
For the reward, we randomly defined a smooth, continuous reward.
To predict the reward function, the rover used a GP with RBF kernel having length-scales of $2$ and a prior variance over the reward of $2$. 
We set the confidence levels as $\alpha_t = 3$ and $\beta_t = 2, \forall t \ge 0$, and the discount factor as $\gamma = 0.9$.

\paragraph{Baselines and metrics.} We used the same baselines and metrics as in the previous synthetic experiment.

\paragraph{Results.} Table~\ref{tab:result} summarizes the results.
The reward was accumulated over the episode, which was normalized with respect to the \textsc{Safe/Reward known} case.
Our \algo\ with either \pes$^2$ or \es$^2$\ outperformed \textsc{SafeMDP} and \textsc{SafeExpOpt-MDP} in terms of the reward. 
This was expected, because SafeMDP does not aim to maximize the cumulative reward, and \textsc{SafeExpOpt-MDP} does not guarantee the near-optimality of the cumulative reward.
Also, \textit{no} unsafe action was executed by any of the tested algorithms.

\begin{table}[t]
\caption{Experimental results with real Mars data.}
\vskip 0.10in
\begin{center}
\begin{small}
\begin{sc}
  \begin{tabular}{*{3}{lrr}}
    \toprule
    & Reward & Unsafe actions \\
    \midrule
    \textbf{\algo} w/ \pes$^2$ & \textbf{0.81} & \textbf{0} \\
    \textbf{\algo} w/ \es$^2$ & \textbf{0.78} & \textbf{0} \\
    \textbf{\algo} & \textbf{0.49} & \textbf{0} \\
    \textsc{\textsc{SafeMDP}} & 0.34 & 0 \\
    \textsc{SafeExpOpt-MDP} & 0.59 & 0 \\
    \textsc{Safe/Reward known} & 1.00 & 0 \\
    \bottomrule
  \end{tabular}
\end{sc}
\end{small}
\end{center}
\label{tab:result}
\end{table}

\section{Conclusion}
We have proposed \algo, a stepwise approach for exploring and optimizing a safety-constrained MDP.
Theoretically, we proved a bound of the sample complexity to achieve $\epsilon_V$-closeness to the optimal policy while guaranteeing safety, with high probability.
We also proposed the \es$^2$ and \pes$^2$ algorithms for improving the efficiency in obtaining rewards.
We developed an open-source environment, \gpsafetygym, to test the effectiveness of \algo~. We also demonstrated the advantages of \algo\ using the real Mars terrain data.

\section*{Acknowledgement}
This work is sponsored by IBM Research AI and Toyota-Tsinghua joint project on health challenges for future city and smart health system design (20203910025).

\bibliographystyle{icml2020}
\bibliography{ref}

\newpage
\onecolumn

\section*{Appendices}

\section*{A. Definitions}

We repeat the relevant definitions in our paper.

\subsection*{A1. Safe Space: \rm{For more details, see \citet{turchetta2016safe}.}}

Set of the states identified as safe up to some confidence level of $\epsilon_g$:
\[
R^\text{safe}_{\epsilon_g}(X) = X \cup \{ \bm{s} \in \mathcal{S} \mid \exists \bm{s}' \in X: g(\bm{s}') -  \epsilon_g - Ld(\bm{s}, \bm{s}') \ge h \}.
\]

Set of states with reachability from $X$:
\[
R_\text{reach}(X) = X \cup \{\bm{s} \in \mathcal{S} \mid \exists \bm{s}' \in X, a \in \mathcal{A}(s'): \bm{s} = f(\bm{s}',a) \}.
\]

Set of states with returnability to $X$:
\begin{alignat*}{2}
&R_{\text{ret}}(X, \bar{X}) = \bar{X} \cup \{\bm{s} \in X \mid \exists a \in\mathcal{A}: f(\bm{s}, a) \in \bar{X} \}, \\
&R^n_\text{ret}(X, \bar{X}) = R_\text{ret}(X, R^{n-1}_\text{ret} (X, \bar{X})), \text{with } R^1_\text{ret}(X, \bar{X}) = R_\text{ret}(X, \bar{X}), \\
&\bar{R}_\text{ret}(X, \bar{X}) = \lim_{n \rightarrow \infty} R^n_\text{ret}(X, \bar{X}).
\end{alignat*}

Set of safe states with reachability and returnability:
\begin{alignat*}{2}
&R_{\epsilon_g}(X) = R^\text{safe}_{\epsilon_g}(X) \cap R_\text{reach}(X) \cap R_{\text{ret}}(R^\text{safe}_{\epsilon_g}(X), X), \\
&R_{\epsilon_g}(X) = R_{\epsilon_g}(R^{n-1}_{\epsilon_g}(X)), \text{with } R^1_{\epsilon_g}(X) = R_{\epsilon_g}(X), \\
&\bar{R}_{\epsilon_g}(X) = \lim_{n \rightarrow \infty} R^n_{\epsilon_g}(X).
\end{alignat*}

Pessimistic safe space:
\begin{alignat*}{2}
&S^-_t = \{ \bm{s} \in \mathcal{S} \mid \exists \bm{s}' \in \mathcal{X}^-_{t-1}: l_t(\bm{s}') - L \cdot d(\bm{s}, \bm{s}') \ge h \}, \\
&\mathcal{X}^-_t = \{ \bm{s} \in S^-_t \mid \bm{s} \in R_\text{reach}(\mathcal{X}^-_{t-1}) \cap \bar{R}_{\text{ret}}(S^-_t, \mathcal{X}^-_{t-1}) \}.
\end{alignat*}

Optimistic safe space:
\begin{alignat*}{2}
&S^+_t = \{ \bm{s} \in \mathcal{S} \mid \exists \bm{s}' \in \mathcal{X}^+_{t-1}: u_t(\bm{s}') - L \cdot d(\bm{s}, \bm{s}') \ge h \}, \\
&\mathcal{X}^+_t = \{ \bm{s} \in S^+_t \mid \bm{s} \in R_\text{reach}(\mathcal{X}^+_{t-1}) \cap \bar{R}_{\text{ret}}(S^+_t, \mathcal{X}^+_{t-1}) \}.
\end{alignat*}

\subsection*{A2. Optimization of Cumulative Reward}

For optimal policy:
\begin{alignat*}{2}
V_{\mathcal{M}}^*(\bm{s}_t) = \max_{s_{t+1} \in R_{\epsilon_g}(S_0)} \left[\ r(\bm{s}_{t+1}) + \gamma  V_{\mathcal{M}}^*(\bm{s}_{t+1}) \ \right].
\end{alignat*}

For balancing exploration and exploitation (neither \es$^2$ nor \pes$^2$ is used):
\begin{alignat*}{2}
&U_t(\bm{s}) = \mu_t^r(\bm{s})  + \alpha^{1/2}_{t+1} \cdot \sigma_t^r(\bm{s}), \\
&J_{\mathcal{X}}^*(\bm{s}_t, \bm{b}^r_t, \bm{b}^g_t) = \max_{s_{t+1} \in \mathcal{X}_{t^*}^-} \bigl[\ U_t(\bm{s}_{t+1}) + \gamma J_{\mathcal{X}}^*(\bm{s}_{t+1}, \bm{b}^r_t, \bm{b}^g_t) \ \bigr].
\end{alignat*}

\subsection*{A3. \es$^2$ Algorithm}

For checking whether the termination condition is satisfied:
\begin{alignat*}{2}
&V_{\mathcal{M}_y}(\bm{s}_t) = \max_{s_{t+1} \in \mathcal{X}_t^+} [\ r'(\bm{s}_{t+1}) + \gamma V_{\mathcal{M}_y}(\bm{s}_{t+1}) \ ], \\
&\mathcal{Y}_t = \{\bm{s}' \in \mathcal{S}^+ \mid \forall \bm{s} \in \mathcal{X}^-_t: \bm{s}' = f(\bm{s}, \pi^*_y(a \mid \bm{s})) \}, \\
&\mathcal{Y}_t \subseteq \mathcal{X}^-_t. 
\end{alignat*}
For balancing exploration and exploitation in terms of reward:
\begin{alignat*}{2}
&J_{\mathcal{Y}}^*(\bm{s}_t, \bm{b}^r_t, \bm{b}^g_t) = \max_{s_{t+1} \in \mathcal{Y}_t} \bigl[\ U_t(\bm{s}_{t+1}) + \gamma J_{\mathcal{Y}}^*(\bm{s}_{t+1}, \bm{b}^r_t, \bm{b}^g_t) \ \bigr].
\end{alignat*}

\subsection*{A4. \pes$^2$ Algorithm}

For checking whether the termination condition is satisfied:
\begin{alignat*}{2}
&V_{\mathcal{M}_z}(\bm{s}_t) = \max_{s_{t+1} \in \mathcal{X}_t^+} [\ P^z\cdot \{ r'(\bm{s}_{t+1}) + \gamma V_{\mathcal{M}_z}(\bm{s}_{t+1}) \} \ ], \\
&\mathcal{Z}_t = \{\bm{s}' \in \mathcal{S}^+ \mid \forall \bm{s} \in \mathcal{X}^-_t: \bm{s}' = f(\bm{s}, \pi_z^*(a \mid \bm{s})) \}, \\
&\mathcal{Z}_t \subseteq \mathcal{X}^-_t. 
\end{alignat*}
For balancing exploration and exploitation in terms of the reward:
\begin{alignat*}{2}
&J_{\mathcal{Z}}^*(\bm{s}_t, \bm{b}^r_t, \bm{b}^g_t) = \max_{s_{t+1} \in \mathcal{Z}_t} \bigl[\ U_t(\bm{s}_{t+1}) + \gamma J_{\mathcal{Z}}^*(\bm{s}_{t+1}, \bm{b}^r_t, \bm{b}^g_t) \ \bigr].
\end{alignat*}

\section*{B. Preliminary Lemma}

\begin{lemma}
\label{lemma:prelim_max_min}
For two arbitrary functions $f_1(x)$ and $f_2(x)$, the following inequality holds:
\begin{alignat*}{2}
\max_x f_1(x) - \max_x f_2(x) \ge \min_x (f_1(x) - f_2(x)).
\end{alignat*}
\end{lemma}

\begin{proof}
For two arbitrary functions $f_4(x)$ and $f_5(x)$, the following inequality holds:
\begin{alignat*}{2}
\max_x f_4(x) + \max_x f_5(x)
&\ \ge \max_x \{f_4(x) + f_5(x)\}.
\end{alignat*}
Let $f_2(x) = f_4(x) + f_5(x)$ and $f_3(x) = -f_4(x)$. Then,
\begin{alignat*}{2}
\max_x \{-f_3(x)\} + \max_x \{f_2(x) + f_3(x)\}
&\ \ge \max_x f_2(x), \\
\max_x \{f_2(x) + f_3(x)\} - \max_x f_2(x)
&\ \ge -\max_x \{-f_3(x)\}, \\
\max_x \{f_2(x) + f_3(x)\} - \max_x f_2(x)
&\ = \min_x f_3(x).
\end{alignat*}
Finally, let $f_1(x) = f_2(x) + f_3(x)$. 
Then, the desired lemma is obtained.
\end{proof}

\section*{C. Near-optimality}
\begin{lemma}
\label{lemma_J_ge_V}
Let $J_{\mathcal{X}}^*(\bm{s}_t, \bm{b}_t^r, \bm{b}_t^g)$ be the value function calculated by \algo\ without the \es$^2$ algorithm. Then, $J_{\mathcal{X}}^*(\bm{s}_t, \bm{b}_t^r, \bm{b}_t^g)$ satisfies the following inequality:
\[
J_{\mathcal{X}}^*(\bm{s}_t, \bm{b}_t^r, \bm{b}_t^g) \ge V^*(\bm{s}_t).
\]
\end{lemma}

\begin{proof}
Consider a state $\bm{s}_t$ and beliefs $\bm{b}_t^r$ and $\bm{b}_t^g$. Also, let $I$ denote the following safety indicator function:
\begin{eqnarray}
I(\bm{s}) :=
  \left\{\begin{array}{ll}
  1 & \quad \text{if}\ \ \bm{s} \in \bar{R}_{\epsilon_g}(S_0), \\
  0 & \quad \text{otherwise}.
  \end{array}
  \right.
  \label{eqn:safety_indicator}
\end{eqnarray}
Then, the following chain of equations and inequalities holds:
\begin{alignat*}{2}
&\ J_{\mathcal{X}}^*(\bm{s}_t, \bm{b}_t^r, \bm{b}_t^g) - V^*(\bm{s}_t) \\
= &\ \max_{s_{t+1} \in \mathcal{X}_{t^*}^-} \left[\ U_t(\bm{s}_{t+1}) + \gamma J_{\mathcal{X}}^*(\bm{s}_{t+1}, \bm{b}_t^r, \bm{b}_t^g) \ \right] - \max_{s_{t+1} \in \bar{R}_{\epsilon_g}(S_0)} \left[\ r(\bm{s}_{t+1}) + \gamma  V_{\mathcal{M}}^*(\bm{s}_{t+1}) \ \right] \\
\ge &\ \max_{s_{t+1} \in \bar{R}_{\epsilon_g}(S_0)} \left[\ U_t(\bm{s}_{t+1}) + \gamma J_{\mathcal{X}}^*(\bm{s}_{t+1}, \bm{b}_t^r, \bm{b}_t^g) \ \right] - \max_{s_{t+1} \in \bar{R}_{\epsilon_g}(S_0)} \left[\ r(\bm{s}_{t+1}) + \gamma  V_{\mathcal{M}}^*(\bm{s}_{t+1}) \ \right] \\
= &\ \max_{a_t} \left[\ I(\bm{s}_{t+1}) \cdot \{U_t(\bm{s}_{t+1}) + \gamma J_{\mathcal{X}}^*(\bm{s}_{t+1}, \bm{b}_t^r, \bm{b}_t^g)\} \ \right] - \max_{a_t} \left[\ I(\bm{s}_{t+1}) \cdot \{r(\bm{s}_{t+1}) + \gamma  V_{\mathcal{M}}^*(\bm{s}_{t+1}) \}\ \right] \\
\ge &\ \min_{a_t} \left[\ I(\bm{s}_{t+1}) \cdot \{ U_t(\bm{s}_{t+1}) - r(\bm{s}_{t+1}) \} + \gamma I(\bm{s}_{t+1}) J_{\mathcal{X}}^*(\bm{s}_{t+1}, \bm{b}_t^r, \bm{b}_t^g) - \gamma I(\bm{s}_{t+1}) V^*(\bm{s}_{t+1}) \ \right] \\
= &\ \min_{a_t} \left[\ I(\bm{s}_{t+1}) \cdot \{ U_t(\bm{s}_{t+1}) - r(\bm{s}_{t+1}) \} + \gamma I(\bm{s}_{t+1}) \{ J_{\mathcal{X}}^*(\bm{s}_{t+1}, \bm{b}_t^r, \bm{b}_t^g) - V^*(\bm{s}_{t+1}) \} \ \right].
\end{alignat*}

The third line follows from $\mathcal{X}_{t^*}^- \supseteq \bar{R}_{\epsilon_g}(S_0)$ in Theorem~\ref{theo:SafeGuarantee}. 
Also, the fourth line follows from the definition of $I$, and the fifth line follows from Lemma~\ref{lemma:prelim_max_min}.
Because $\bm{s}$ is arbitrary in the above derivation, we have
\begin{alignat*}{2}
\min_{\bm{s}_t} [\ J_{\mathcal{X}}^*(\bm{s}_t, \bm{b}_t^r, \bm{b}_t^g) - V^*(\bm{s}_t)\ ] 
\ge \min_{\bm{s}_{t+1}} \left[\ I(\bm{s}_{t+1}) \{ U_t(\bm{s}_{t+1}) - r(\bm{s}_{t+1}) \} + \gamma I(\bm{s}_{t+1}) \{J^*(\bm{s}_{t+1}, \bm{b}_t^r, \bm{b}_t^g) - V^*(\bm{s}_{t+1}) \}\ \right].
\end{alignat*}
By Lemma~\ref{lemma:reward_confidence}, the following equation holds with probability at least $1-\Delta^r$:
\begin{alignat*}{2}
\min_{\bm{s}_t} [\ J_{\mathcal{X}}^*(\bm{s}_t, \bm{b}_t^r, \bm{b}_t^g) - V^*(\bm{s}_t, \bm{b}_t^r, \bm{b}_t^g)\ ] \ge  \gamma \cdot \min_{\bm{s}_{t+1}} \left[ I(\bm{s}_{t+1}) \{ J_{\mathcal{X}}^*(\bm{s}_{t+1}, \bm{b}_t^r, \bm{b}_t^g) - V^*(\bm{s}_{t+1})\}\ \right]
\end{alignat*}
Repeatedly applying this equation proves the desired lemma. 
Therefore, we have
\[
J_{\mathcal{X}}^*(\bm{s}_t, \bm{b}_t^r, \bm{b}_t^g) \ge V^*(\bm{s}_t)
\]
with high probability.
\end{proof}

\begin{lemma}
\label{lemma_4}
{\rm\textbf{(Generalized induced inequality)}} Let $\bm{b}^r, \bm{b}^g, r$ and $\hat{\bm{b}}^r, \hat{\bm{b}}^g, \hat{r}$ be the beliefs (over reward and safety, respectively) and reward functions (including the exploration bonus) that are identical on some set of states $\Omega$ --- i.e., $\bm{b}^r = \hat{\bm{b}}^r$, $\bm{b}^g = \hat{\bm{b}}^g$, and $r = \hat{r}$ for all $\bm{s} \in \Omega$. Let $P(A_{\Omega})$ be the probability that a state not in $\Omega$ is generated when starting from state $\bm{s}$ and following a policy~$\pi$. If the value is bound in $[0, V_{\max}]$, then
\[
V^\pi(\bm{s}, \bm{b}^r, \bm{b}^g, r) \ge V^\pi(\bm{s}, \hat{\bm{b}}_r, \hat{\bm{b}}_g, \hat{r}) - V_{\max} P(A_{\Omega}),
\]
where we now make explicit the dependence of the value function on the reward.
\end{lemma}

\begin{proof}
The lemma follows from Lemma~8 in \citet{strehl2005theoretical}.
\end{proof}

\begin{lemma}
\label{lemma_Sigma_r}
Assume that the reward function $r$ satisfies $\|r\|_k^2 \le B^r$, and that the noise $n_t^r$ is $\sigma_r$-sub-Gaussian. If $\alpha_t = B^r + \sigma_r \sqrt{2(\Gamma^r_{t-1} + 1 + \log(1/\Delta^r))}$ and $C_r = 8/\log(1+\sigma_r^{-2})$, then the following holds:
\[
\frac{1}{2}\sqrt{\frac{C_r \alpha_{t^*} \Gamma_{t^*}^r}{t^*}} \ge \alpha_{t^*}^{1/2} \sigma_{t^*}^r(\bm{s}),
\]
with probability at least $1 - \Delta^r$.
\end{lemma}

\begin{proof}
The lemma follows from Lemma~4 in \citet{chowdhury2017kernelized}.
\end{proof}

\section*{D. \es$^2$ algorithm}

\begin{lemma}
\label{lemma_s_t_in_Y}
Assume that $\mathcal{Y}_t \subseteq \mathcal{X}_t^-$ holds. Suppose that we obtain the optimal policy, $\pi^*_y$ on the basis of $J_{\mathcal{Y}}^*(\bm{s}_t, \bm{b}^r_t, \bm{b}^g_t) = \max_{s_{t+1} \in \mathcal{Y}_t} \bigl[U_t(\bm{s}_{t+1}) + \gamma J_{\mathcal{Y}}^*(\bm{s}_{t+1}, \bm{b}^r_t, \bm{b}^g_t) \bigr]$. Then, for all $t$, the following holds:
\[
\bm{s}_t \in \mathcal{Y}_t \Longrightarrow \bm{s}_{t+1} \in \mathcal{Y}_t.
\]
\end{lemma}

\begin{proof}
When $\mathcal{Y}_t \subseteq \mathcal{X}_t^-$ holds, we have 
\begin{alignat*}{2}
\{\bm{s}' \in \mathcal{S}^+ \mid \forall \bm{s} \in \mathcal{Y}_t: \bm{s}' = f(\bm{s}, \pi_y^*(a \mid \bm{s})) \}
&\subseteq
\{\bm{s}' \in \mathcal{S}^+ \mid \forall \bm{s} \in \mathcal{X}^-_t: \bm{s}' = f(\bm{s}, \pi_y^*(a \mid \bm{s})) \} \\
&=
\mathcal{Y}_t.
\end{alignat*}
This means that the next state $\bm{s}_{t+1}$ will be within $\mathcal{Y}_t$ if the agent is in $\mathcal{Y}_t$ and decides the action based on $\pi_y^*$. 
Therefore, we have the desired lemma.
\end{proof}

\begin{lemma}
\label{lemma_J_X_J_Y}
Assume that $\mathcal{Y}_t \subseteq \mathcal{X}^-_t$ holds, and let $J_{\mathcal{Y}}^*(\bm{s}_t, \bm{b}_t^r, \bm{b}_t^g)$ be the value function calculated by \algo\ with the \es$^2$ algorithm. 
Then, for all $\bm{s}_t \in \mathcal{X}^-_t$, $J_{\mathcal{Y}}^*(\bm{s}_t, \bm{b}_t^r, \bm{b}_t^g)$ satisfies the following equation:
\[
J_{\mathcal{Y}}^*(\bm{s}_t, \bm{b}_t^r, \bm{b}_t^g) \ge V^*(\bm{s}_t).
\]
\end{lemma}

\begin{proof}
Consider a state $\bm{s}_t \in \mathcal{X}^-_t$ and beliefs $\bm{b}^r$ and $\bm{b}^g$. 
Also, we define the function $I$ as in (\ref{eqn:safety_indicator}).
Then, the following chain of the equations and inequalities holds:
\begin{alignat*}{2}
&\ J_{\mathcal{Y}}^*(\bm{s}_t, \bm{b}^r_t, \bm{b}^g_t) - V^*(\bm{s}_t) \\
=&\ \max_{s_{t+1} \in \mathcal{Y}_t} \bigl[\ U_t(\bm{s}_{t+1}) + \gamma J_{\mathcal{Y}}^*(\bm{s}_{t+1}, \bm{b}^r_t, \bm{b}^g_t) \ \bigr] - \max_{a_t} \left[\ I(\bm{s}_{t+1}) \cdot \{r(\bm{s}_{t+1}) + \gamma  V_{\mathcal{M}}^*(\bm{s}_{t+1}) \}\ \right] \\
=&\ \max_{s_{t+1} \in \mathcal{Y}_t} \bigl[\ U_t(\bm{s}_{t+1}) + \gamma J_{\mathcal{Y}}^*(\bm{s}_{t+1}, \bm{b}^r_t, \bm{b}^g_t) \ \bigr] - \max_{s_{t+1} \in \mathcal{X}^+_t} \left[\ I(\bm{s}_{t+1}) \cdot \{r(\bm{s}_{t+1}) + \gamma  V_{\mathcal{M}}^*(\bm{s}_{t+1}) \}\ \right] \\
=&\ \max_{s_{t+1} \in \mathcal{Y}_t} \bigl[\ U_t(\bm{s}_{t+1}) + \gamma J_{\mathcal{Y}}^*(\bm{s}_{t+1}, \bm{b}^r_t, \bm{b}^g_t) \ \bigr] - \max_{s_{t+1} \in \mathcal{Y}_t} \left[\ I(\bm{s}_{t+1}) \cdot \{r(\bm{s}_{t+1}) + \gamma  V_{\mathcal{M}}^*(\bm{s}_{t+1}) \}\ \right] \\
\ge&\ \min_{s_{t+1} \in \mathcal{Y}_t} \bigl[\ U_t(\bm{s}_{t+1}) + \gamma J_{\mathcal{Y}}^*(\bm{s}_{t+1}, \bm{b}^r_t, \bm{b}^g_t) - I(\bm{s}_{t+1}) \cdot \{r(\bm{s}_{t+1}) + \gamma  V_{\mathcal{M}}^*(\bm{s}_{t+1}) \}\ \bigr] \\
\ge&\ \min_{s_{t+1} \in \mathcal{Y}_t} \bigl[\ U_t(\bm{s}_{t+1}) + \gamma J_{\mathcal{Y}}^*(\bm{s}_{t+1}, \bm{b}^r_t, \bm{b}^g_t) - \{r(\bm{s}_{t+1}) + \gamma  V_{\mathcal{M}}^*(\bm{s}_{t+1}) \}\ \bigr] \\
=&\ \min_{s_{t+1} \in \mathcal{Y}_t} \bigl[\ U_t(\bm{s}_{t+1}) - r(\bm{s}_{t+1}) + \gamma J_{\mathcal{Y}}^*(\bm{s}_{t+1}, \bm{b}^r_t, \bm{b}^g_t) - \gamma  V_{\mathcal{M}}^*(\bm{s}_{t+1}) \ \bigr].
\end{alignat*}
The second and third lines follow from the definitions of $I$ and $V_{\mathcal{M}}^*$. 
The forth line follows from the definition of $\mathcal{Y}$ and the assumption of $\mathcal{Y}_t \subseteq \mathcal{X}^-_t$. 
The fifth line follows from Lemma~\ref{lemma:prelim_max_min}.

Then, by Lemma~\ref{lemma:reward_confidence}, the following equation holds with probability at least $1-\Delta^r$:
\begin{alignat*}{2}
\min_{s_t \in \mathcal{X}^-_t} \bigl[\ J_{\mathcal{Y}}^*(\bm{s}_t, \bm{b}^r_t, \bm{b}^g_t) - V^*(\bm{s}_t) \} \ \bigr]
&\ge \gamma \cdot \min_{s_{t+1} \in \mathcal{Y}_t} \bigl[\ J_{\mathcal{Y}}^*(\bm{s}_{t+1}, \bm{b}^r_t, \bm{b}^g_t) - V_{\mathcal{M}}^*(\bm{s}_{t+1}) \ \bigr] \\
&\ge \gamma^2 \cdot \min_{s_{t+2} \in \mathcal{Y}_t} \bigl[\ J_{\mathcal{Y}}^*(\bm{s}_{t+2}, \bm{b}^r_t, \bm{b}^g_t) - V_{\mathcal{M}}^*(\bm{s}_{t+2}) \ \bigr].
\end{alignat*}
The second line follows from Lemma~\ref{lemma_s_t_in_Y}. 
Repeatedly applying this equation proves the desired lemma. 
Therefore, for all $\bm{s}_t \in \mathcal{X}_t^-$, we have
\[
J_{\mathcal{Y}}^*(\bm{s}_t, \bm{b}_t^r, \bm{b}_t^g) \ge V^*(\bm{s}_t).
\]
\end{proof}

\section*{E. Main Theoretical Results}

\textbf{Theorem 1.} \textit{Assume that the safety function $g$ satisfies $\|g\|_k^2 \le B^g$ and is $L$-Lipschitz continuous.
Also, assume that $S_0 \ne \emptyset$ and $g(\bm{s}) \ge h$ for all $\bm{s} \in S_0$.
Fix any $\epsilon_g > 0$ and $\Delta^g \in (0, 1)$.
Suppose that we conduct the stage of ``exploration of safety'' with the noise $n_t^g$ being $\sigma_g$-sub-Gaussian, and that $\beta_t = B^g + \sigma_g \sqrt{2(\Gamma^g_{t-1} + 1 + \log(1/\Delta^g))}$ until $\max_{\bm{s} \in G_t} w_t(\bm{s}) < \epsilon_g$ is achieved.
Finally, let $t^*$ be the smallest integer satisfying
\[
\frac{t^*}{\beta_{t^*}\Gamma^g_{t^*}} \ge \frac{C_g|\bar{R}_0(S_0)|}{\epsilon_g^2} \cdot D(\mathcal{M}),
\]
with $C_g=8/\log(1+\sigma_g^{-2})$.
Then, the following statements jointly hold with probability at least $1 - \Delta^g$:
\begin{itemize}
    \item $\forall t \ge 1$, $g(\bm{s}_t) \ge h$,
    \item $\exists t_0 \le t^*$, $\bar{R}_{\epsilon_g} (S_0) \subseteq \mathcal{X}^-_{t_0} \subseteq \bar{R}_0 (S_0)$.
\end{itemize}}

\begin{proof}
This is an extension of Theorem 1 in \citet{turchetta2016safe} to our settings, where $t$ represents not the number of samples but the number of actions.
\end{proof}

\textbf{Theorem 2.} \textit{Assume that the reward function $r$ satisfies $\|r\|_k^2 \le B^r$, and that the noise is $\sigma_r$-sub-Gaussian. 
Let $\pi_t$ denote the policy followed by \algo\ at time $t$, and let $\bm{s}_t$ and $\bm{b}^r_t, \bm{b}^g_t$ be the corresponding state and beliefs, respectively.
Let $t^*$ be the smallest integer satisfying
$\frac{t^*}{\beta_{t^*}\Gamma^g_{t^*}} \ge \frac{C_g|\bar{R}_0(S_0)|}{\epsilon_g^2} D(\mathcal{M})$, and fix any $\Delta^r \in (0, 1)$.
Finally, set $\alpha_t = B^r + \sigma_r \sqrt{2(\Gamma^r_{t-1} + 1 + \log(1/\Delta^r))}$ and
\[
\epsilon^*_V = V_{\max} \cdot (\Delta^g + \Sigma_{t^*}^r/R_{\max}),
\]
with $\Sigma^r_{t^*} =  \frac{1}{2}\sqrt{\frac{C_r \alpha_{t^*} \Gamma_{t^*}^r}{t^*}}$.
Then, with high probability,
\[
V^{\pi_t}(\bm{s}_t, \bm{b}^r_t, \bm{b}^g_t) \ge V^*(\bm{s}_t) - \epsilon^*_V
\]
--- i.e., the algorithm is $\epsilon^*_V$-close to the optimal policy --- for all but $t^*$ time steps, while guaranteeing safety with probability at least $1-\Delta^g$.}

\begin{proof}
Define $\tilde{r}$ as the reward function (including the exploration bonus) that is used by \algo.
Let $\hat{r}$ be a reward function equal to $r$ on $\Omega$ and equal to $\tilde{r}$ elsewhere.
Furthermore, let $\tilde{\pi}$ be the policy followed by \algo\ at time $t$, that is, the policy calculated on the basis of the current beliefs, (i.e., $\bm{b}^r_t$ and $\bm{b}^g_t$) and the reward $\tilde{r}$. Finally, let $A_{\Omega}$ be the event in which $\tilde{\pi}$ escapes from $\Omega$. Then,
\[
V^{\pi_t}(r, \bm{s}_t, \bm{b}^r_t, \bm{b}^g_t) \ge V^{\tilde{\pi}}(\hat{r}, \bm{s}_t, \bm{b}^r_t, \bm{b}^g_t) - V_{\max} P(A_{\Omega})
\]
by Lemma~\ref{lemma_4}.
In addition, note that, for all $t \ge t^*$, because $\hat{r}$ and $\tilde{r}$ differ by at most $\alpha_{t^*}^{1/2} \sigma^r_{t^*}$ at each state,
\begin{alignat}{2}
\label{eq:bar_eta_R_EB}
|V^{\tilde{\pi}}(\hat{r}, \bm{s}_t, \bm{b}^r_t, \bm{b}^g_t) - V^{\tilde{\pi}}(\tilde{r}, \bm{s}_t, \bm{b}^r_t, \bm{b}^g_t) |
& \le \frac{1}{1-\gamma} \cdot \alpha_{t^*}^{1/2} \sigma_{t^*}^r(\bm{s}) \nonumber \\
& \le V_{\max}/R_{\max} \cdot \Sigma^r_{t^*}.
\end{alignat}
For the above inequality, we used Lemma~\ref{lemma_Sigma_r}.
Here, consider the case of $\Omega =\mathcal{X}_{t^*}^-$. Once the safe region is fully explored, $P(A_{\Omega}) \le \Delta^g$ holds after
$t^*$ time steps. Then, the following chain of equations and inequalities holds:
\begin{alignat*}{2}
V^{\pi_t}(R, \bm{s}, \bm{b})
\ge&\ V^{\tilde{\pi}}(\hat{R}, \bm{s}, \bm{b}) - V_{\max} \cdot P(A_{\Omega}) \\
=&\ V^{\tilde{\pi}}(\hat{R}, \bm{s}, \bm{b}) - V_{\max} \cdot P(A_{\mathcal{X}^-}) \\
\ge&\ V^{\tilde{\pi}}(\hat{R}, \bm{s}, \bm{b}) - V_{\max} \cdot \Delta^g \\
\ge&\ V^{\tilde{\pi}}(\tilde{R}, \bm{s}, \bm{b}) - V_{\max} \cdot (\Delta^g + \Sigma_{t^*}^r/R_{\max})  \\
=&\ J_{\mathcal{X}}^*(\tilde{R}, \bm{s}, \bm{b}) - V_{\max} \cdot (\Delta^g + \Sigma_{t^*}^r/R_{\max}) \\
\ge&\ V^*(R, \bm{s}) - V_{\max} \cdot (\Delta^g + \Sigma_{t^*}^r/R_{\max}).
\end{alignat*}
In this derivation, the second line follows from the assumption of $\Omega = \mathcal{X}^-$, the third line follows from $P(A_{\mathcal{X}^-}) \le \Delta^g$, the fourth line follows from (\ref{eq:bar_eta_R_EB}), the fifth line follows from the fact that $\tilde{\pi}$ is precisely the optimal policy for $\tilde{R}$ and $\bm{b}$, and the final line follows from Lemma~\ref{lemma_J_ge_V}.
\end{proof}

\textbf{Theorem 3.} \textit{Assume that the reward function $r$ satisfies $\|r\|_k^2 \le B^r$, and that the noise is $\sigma_r$-sub-Gaussian.
Let $\pi_t$ denote the policy followed by \algo\ with the the \es$^2$ algorithm at time $t$, and let $\bm{s}_t$ and $\bm{b}^r_t, \bm{b}^g_t$ be the corresponding state and beliefs, respectively.
Let $\tilde{t}$ be the smallest integer for which (\ref{eq:stopping_cond}) holds, and fix any $\Delta^r \in (0, 1)$.
Finally, set $\alpha_t = B^r + \sigma_r \sqrt{2(\Gamma^r_{t-1} + 1 + \log(1/\Delta^r))}$ and 
\[
\tilde{\epsilon}_V = V_{\max} \cdot (\Delta^g + \Sigma_{\tilde{t}}^r/R_{\max}),
\]
with $\Sigma^r_{\tilde{t}} = \frac{1}{2} \sqrt{\frac{C_r \alpha_{\tilde{t}} \Gamma_{\tilde{t}}^r}{\tilde{t}}}$.
Then, with high probability,
\[
V^{\pi_t}(\bm{s}_t, \bm{b}^r_t, \bm{b}^g_t) \ge V^*(\bm{s}_t) - \tilde{\epsilon}_V
\]
--- i.e., the algorithm is $\tilde{\epsilon}_V$-close to the optimal policy --- for all but $\tilde{t}$ time steps while guaranteeing safety with probability at least $1-\Delta^g$.}

\begin{proof}
The proof of Theorem~\ref{theo:es2} is analogous to that of Theorem~\ref{theo:NearOpt}.
Define $\tilde{r}$ as the reward function (including the exploration bonus) that is used by \algo.
Let $\hat{r}$ be a reward function equal to $r$ on $\mathcal{Y}$ and equal to $\tilde{r}$ elsewhere.
Furthermore, let $\tilde{\pi}$ be the policy followed by \algo\ with the \es$^2$ algorithm at time $t$, that is, the policy calculated on the basis of the current beliefs, (i.e., $\bm{b}^r_t$ and $\bm{b}^g_t$) and the reward $\tilde{r}$. 
Finally, let $A_{\mathcal{Y}}$ be the event in which $\tilde{\pi}$ escapes from $\mathcal{Y}$. Then,
\[
V^{\pi_t}(r, \bm{s}_t, \bm{b}^r_t, \bm{b}^g_t) \ge V^{\tilde{\pi}}(\hat{r}, \bm{s}_t, \bm{b}^r_t, \bm{b}^g_t) - V_{\max} P(A_{\mathcal{Y}})
\]
by Lemma~\ref{lemma_4}.
In addition, note that, for all $t \ge \tilde{t}$, because $\hat{r}$ and $\tilde{r}$ differ by at most $\alpha_{\tilde{t}}^{1/2} \sigma^r_{\tilde{t}}$ at each state,
\begin{alignat}{2}
\label{eq:bar_eta_R_EB_es2}
|V^{\tilde{\pi}}(\hat{r}, \bm{s}_t, \bm{b}^r_t, \bm{b}^g_t) - V^{\tilde{\pi}}(\tilde{r}, \bm{s}_t, \bm{b}^r_t, \bm{b}^g_t) |
& \le \frac{1}{1-\gamma} \cdot \alpha_{\tilde{t}}^{1/2} \sigma_{\tilde{t}}^r(\bm{s}) \nonumber \\
& \le V_{\max}/R_{\max} \cdot \Sigma^r_{\tilde{t}}.
\end{alignat}
For the above inequalities, we used Lemma~\ref{lemma_Sigma_r}.
Then, the following chain of equations and inequalities holds:
\begin{alignat*}{2}
V^{\pi_t}(R, \bm{s}, \bm{b})
=&\ V^{\tilde{\pi}}(\hat{R}, \bm{s}, \bm{b}) - V_{\max} \cdot P(A_{\mathcal{Y}}) \\
\ge&\ V^{\tilde{\pi}}(\hat{R}, \bm{s}, \bm{b}) - V_{\max} \cdot \Delta^g \\
\ge&\ V^{\tilde{\pi}}(\tilde{R}, \bm{s}, \bm{b}) - V_{\max} \cdot (\Delta^g + \Sigma_{\tilde{t}}^r/R_{\max})   \\
=&\ J_{\mathcal{Y}}^*(\tilde{R}, \bm{s}, \bm{b}) - V_{\max} \cdot (\Delta^g + \Sigma_{\tilde{t}}^r/R_{\max}) \\
\ge&\ V^*(R, \bm{s}) - V_{\max} \cdot (\Delta^g + \Sigma_{\tilde{t}}^r/R_{\max}).
\end{alignat*}
In this derivation, the second line follows from $P(A_{\mathcal{Y}}) \le \Delta^g$, the third line follows from (\ref{eq:bar_eta_R_EB_es2}), the fourth line follows from the fact that $\tilde{\pi}$ is precisely the optimal policy for $\tilde{R}$ and $\bm{b}$, and the final line follows from Lemma~\ref{lemma_J_X_J_Y}.
\end{proof}

\end{document}